\documentclass[twoside,11pt]{article}

\usepackage{jmlr2e}

\usepackage{amsmath}
\usepackage{hyperref}
\usepackage{caption}
\usepackage{subfig}
\usepackage{bm}

\DeclareMathOperator{\E}{\mathbb{E}}
\DeclareMathOperator*{\argmax}{arg\,max}

\DeclareMathOperator*{\Int}{Int}
\DeclareMathOperator*{\KL}{KL}
\DeclareMathOperator*{\Cone}{Cone}
\DeclareMathOperator*{\Hull}{Hull}
\DeclareMathOperator*{\diag}{diag}

\DeclareMathOperator*{\R}{\mathbb{R}}

\DeclareMathOperator*{\Z}{\mathbb{Z}}
\DeclareMathOperator*{\Y}{\mathbb{Y}}
\DeclareMathOperator*{\X}{\mathbb{X}}
\DeclareMathOperator*{\W}{\mathbb{W}}

\newcommand{\Yhist}[1]{\mathcal{Y}_{#1}}
\newcommand{\bmYhist}[1]{\bm{\mathcal{Y}}_{#1}}

\newtheorem{assm}{\assmname}

\makeatletter
\newenvironment{assume}[1]
  {\def\assmname{#1}%
   \refstepcounter{assm}%
   \assm\def\@currentlabel{#1}}
  {\endassm}
\makeatother

\jmlrheading{}{2017}{}{2/17}{}{}{Stephen N. Pallone, Peter I. Frazier
  and Shane G. Henderson}

\ShortHeadings{Entropy Pursuit for Active Preference Learning}{Pallone, Frazier,
  and Henderson}
\firstpageno{1}

\begin{document}

\title{Bayes-Optimal Entropy Pursuit for Active Choice-Based Preference Learning}

\author{\name Stephen N. Pallone \email snp32@cornell.edu \\
        \name Peter I. Frazier \email pf98@cornell.edu \\
        \name Shane G. Henderson \email sgh9@cornell.edu \\\\
        \addr School of Operations Research and Information
        Engineering \\
        290 Rhodes Hall, Cornell University \\
        Ithaca, NY 14853, USA}

\editor{}
\maketitle

\begin{abstract}
  We analyze the problem of learning a single user's preferences in an
  active learning setting, sequentially and adaptively querying the
  user over a finite time horizon. Learning is conducted via
  choice-based queries, where the user selects her preferred option
  among a small subset of offered alternatives. These queries have
  been shown to be a robust and efficient way to learn an individual's
  preferences. We take a parametric approach and model the user's
  preferences through a linear classifier, using a Bayesian prior to
  encode our current knowledge of this classifier. The rate at which
  we learn depends on the alternatives offered at every time
  epoch. Under certain noise assumptions, we show that the
  Bayes-optimal policy for maximally reducing entropy of the posterior
  distribution of this linear classifier is a greedy policy, and that
  this policy achieves a linear lower bound when alternatives can be
  constructed from the continuum. Further, we analyze a different
  metric called misclassification error, proving that the performance
  of the optimal policy that minimizes misclassification error is
  bounded below by a linear function of differential entropy. Lastly,
  we numerically compare the greedy entropy reduction policy with a
  knowledge gradient policy under a number of scenarios, examining
  their performance under both differential entropy and
  misclassification error.
\end{abstract}



\section{Introduction}

The problem of preference learning is a well-studied and widely
applicable area of study in the machine learning
literature. Preference elicitation is by no means a new problem
\citep{schapire1998learning}, and is now ubiquitous in many different
forms in nearly all subfields of machine learning. One such scenario
is the active learning setting, where one sequentially and adaptively
queries the user to most efficiently learn his or her preferences. In
general, learning in an online setting can be more efficient than
doing so in an offline supervised learning setting, which is
consequential when queries are expensive. This is often the case for
preference elicitation, where a user may not be inclined to answer too
many questions. The ability to adaptively query the user with
particular exemplars that facilitate learning to the labels of the
rest is invaluable in the context of preference elicitation.

In particular, there is great interest in using choice-based queries
to learn the preferences of an individual user. In this setting, a
user is offered two or more alternatives and is asked to select the
alternative he or she likes most. There are other types of responses
that can assess one's preferences among a set of alternatives, such as
rating each of the items on a scale, or giving a full preference order
for all alternatives in the set. However, choosing the most-preferred
item in a given set is a natural task, and is a more robust
measurement of preference than rating or fully-ranking items. For this
reason, choice-based methods have been shown to work well in practice
\citep[see][]{louviere2000stated}, and these are the types of queries
we study. In this paper, we formulate the problem of sequential
choice-based preference elicitation as a finite horizon adaptive
learning problem.

The marketing community has long been focused on preference
elicitation and isolating features that matter the most to
consumers. In this field, \textit{conjoint analysis} is a class of
methods that attempts to learn these important features by offering
users a subset of alternatives \citep{green1978conjoint}. Lately,
there has been a push in the marketing community to design sequential
methods that adaptively select the best subset of alternatives to
offer the user. In the marketing research literature, this is referred
to as adaptive choice-based conjoint analysis. In the past,
geometrically-motivated heuristics have been used to adaptively choose
questions \citep{toubia2004polyhedral}. These heuristics have since
evolved to include probabilistic modeling that captures the
uncertainty in user responses \citep{toubia2007probabilistic}.

These problems are also tackled by the active learning community. For
instance, \citet{maldonado2015advanced} use existing support vector
machine (SVM) technology to identify features users find important.
In the context of preference elicitation in the active learning
literature, there are two main approaches. The first is to take a
non-parametric approach and infer a full preference ranking, labeling
every pairwise combination of alternatives
\citep{furnkranz2003pairwise}. The benefit to this approach is the
generality offered by a non-parametric model and its ability to
capture realistic noise.  Viewing preference learning as a generalized
binary search problem, \citet{nowak2011geometry} proves exponential
convergence in probability to the correct preferential ordering for
all alternatives in a given set, and shows his algorithm is optimal to
a constant factor.  Unfortunately, this probabilistic upper bound is
weakened by a coefficient that is quadratic in the total number of
alternatives, and the running time of this optimal policy is
proportional to the number of valid preferential orderings of all the
alternatives. These issues are common for non-parametric ranking
models. Using a statistical learning theoretic framework,
\citet{ailon2012active} develops an adaptive and computationally
efficient algorithm to learn a ranking, but the performance guarantees
are only asymptotic. In practice, one can only expect to ask a user a
limited number of questions, and in this scenario,
\citet{yu2012comparison} show that taking a Bayesian approach to
optimally and adaptively selecting questions is indispensable to the
task of learning preferences for a given user. In the search for
finite-time results and provable bounds, we opt to learn a parametric
model using a Bayesian approach. In particular, this paper focuses on
a greedy policy that maximally reduces posterior entropy of a linear
classifier, leveraging information theory to derive results pertaining
to this policy.

Maximizing posterior entropy reduction has long been a suggested
objective for learning algorithms
\citep{lindley1956measure,bernardo1979expected}, especially within the
context of active learning \citep{mackay1992information}. But even
within this paradigm of preference elicitation, there is a variety of
work that depends on the user response model. For example,
\citet{dzyabura2011active} study maximizing entropy reduction under
different response heuristics, and \citet{saure2016ellipsoidal} uses
ellipsoidal credibility regions to capture the current state of
knowledge of a user's preferences. Using an entropy-based objective
function allows one to leverage existing results in information theory
to derive theoretical finite-time guarantees
\citep{jedynak2012twenty}. Most similar to our methodology,
\citet{brochu2010bayesian} and \citet{houlsby2011bayesian} model a
user's utility function using a Gaussian process, updating the
corresponding prior after each user response, and adaptively choose
questions by minimizing an estimate of posterior entropy. However,
while the response model is widely applicable and the method shows
promise in practical situations, the lack of theoretical guarantees
leaves much to be desired. Ideally, one would want concrete
performance bounds for an entropy-based algorithm under a
parameterized response model. In contrast, this paper proves
information theoretic results in the context of adaptive choice-based
preference elicitation for arbitrary feature-space dimension,
leverages these results to derive bounds for performance, and shows
that a greedy entropy reduction policy (hereafter referred to as
\textit{entropy pursuit}) optimally reduces posterior entropy of a
linear classifier over the course of multiple choice-based
questions. In particular, the main contributions of the paper are
summarized as follows:
\begin{itemize}
  \item In Section~\ref{sec:problem spec}, we formally describe the
    response model for the user. For this response model, we prove a
    linear lower bound on the sequential entropy reduction over a
    finite number of questions in Section~\ref{sec:entropy}, and
    provide necessary and sufficient conditions for asking an optimal
    comparative question.

  \item Section~\ref{subsec:continuum} presents results showing that
    the linear lower bound can be attained by a greedy algorithm up to
    a multiplicative constant when we are allowed to fabricate
    alternatives (i.e., when the set of alternatives has a non-empty
    interior). Further, the bound is attained exactly with moderate
    conditions on the noise channel.

  \item Section~\ref{sec:misclass} focuses on misclassification
    error, a more intuitive metric of measuring knowledge of a user's
    preferences. In the context of this metric, we show a Fano-type
    lower bound on the optimal policy in terms of an increasing linear
    function of posterior differential entropy. 
    
  \item Finally in Section~\ref{sec:computation}, we provide numerical
    results demonstrating that entropy pursuit performs similarly to
    an alternative algorithm that greedily minimizes misclassification
    error. This is shown in a variety of scenarios and across both
    metrics. Taking into account the fact that entropy pursuit is far
    more computationally efficient than the alternative algorithm, we
    conclude that entropy pursuit should be preferred in practical
    applications.
\end{itemize}


\section{Problem Specification}
\label{sec:problem spec}

The alternatives $x^{(i)} \in \mathbb{R}^d$ are represented by
$d$-dimensional feature vectors that encode all of their
distinguishing aspects. Let $\X$ be the set of all such
alternatives. Assuming a linear utility model, each user has her own
linear classifier $\bm{\theta}\in\Theta\subset\mathbb{R}^d$ that
encodes her preferences \footnote{Throughout the paper, we use
  boldface to denote a random variable.}. At time epoch $k$, given $m$
alternatives $X_k = \{x^{(1)}_k,x_k^{(2)},\dots,x_k^{(m)}\} \in \X^m$,
the user prefers to choose the alternative $i$ that maximizes
$\bm\theta^T x_k^{(i)}$. However, we do not observe this preference
directly. Rather, we observe a signal influenced by a noise
channel. In this case, the signal is the response we observe from the
user.

Let $\Z=\{1,2,\dots,m\}$ denote the $m$ possible alternatives. We
define $\bm Z_k(X_k)$ to be the alternative that is consistent with
our linear model after asking question $X_k$, that is, $\bm Z_k(X_k)=
\min\, \left\{\argmax_{i \in \Z} \bm \theta^T x_k^{(i)} \right\}$. The
minimum is just used as a tie-breaking rule; the specific rule is not
important so long as it is deterministic. We do not observe $\bm
Z_k(X_k)$, but rather observe a signal $\bm Y_k(X_k) \in \Y$, which
depends on $\bm Z_k(X_k)$. We allow $\Y$ to characterize any type of
signal that can be received from posing questions in $\X$. In general,
the density of the conditional distribution of $\bm Y_k(X_k)$ given
$\bm Z_k(X_k)=z$ is denoted $f^{(z)}(\cdot)$. In this paper, we
primarily consider the scenario in which $\Y = \Z = \{1,2,\dots,m\}$,
where nature randomly perturbs $\bm Z_k(X_k)$ to some (possibly the
same) element in $\Z$. In this scenario, the user's response to the
preferred alternative is the signal $\bm Y_k(X_k)$, which is observed
in lieu of the model-consistent ``true response'' $\bm Z_k(X_k)$. In
this case, we define a noise channel stochastic matrix $P$ by setting
$P^{(zy)} = f^{(z)}(y)$ to describe what is called a \textit{discrete
  noise channel}.

One sequentially asks the user questions and learns from each of their
responses. Accordingly, let $\mathbb{P}_k$ be the probability measure
conditioned on the $\sigma$-field generated by $\bmYhist{k} =
\left(\bm Y_\ell(X_\ell):\,1\leq \ell\leq k-1\right)$. Similarly, let
$\Yhist{k} = \left\{ Y_\ell(X_\ell):\, 1 \leq \ell \leq k-1\right\}$
denote the history of user responses. As we update, we condition on
the previous outcomes, and subsequently choose a question $X_k$ that
depends on all previous responses $\Yhist{k}$ from the
user. Accordingly, let policy $\pi$ return a comparative question $X_k
\in \X^m$ that depends on time epoch $k$ and past response history
$\Yhist{k}$. The selected question $X_k$ may also depend on
i.i.d. random uniform variables, allowing for stochastic policies. We
denote the space of all such policies $\pi$ as $\Pi$. In this light,
let $\mathbb{E}^\pi$ be the expectation operator induced by policy
$\pi$.

In this paper, we consider a specific noise model, which is
highlighted in the following assumptions.
\begin{assume}{Noise Channel Assumptions}
  \label{assm:noise}
  For every time epoch $k$, signal $\bm Y_k(X_k)$ and true response
  $\bm Z_k(X_k)$ corresponding to comparative question $X_k$, we
  assume
  \begin{itemize}
    \item model-consistent response $\bm Z_k(X_k)$ is a deterministic
      function of question $X$ and linear classifier $\bm{\theta}$,
      and
    \item given true response $\bm Z_k(X_k)$, signal $\bm Y_k(X_k)$ is
      conditionally independent of linear classifier $\bm{\theta}$ and
      previous history $\Yhist{k}$, and
    \item the conditional densities $f = \{f^{(z)}:\,z\in\Z\}$ differ
      from each other on a set of Lebesgue measure greater than zero.
  \end{itemize}
\end{assume}
The first two assumptions ensure that all the information regarding
$\bm{\theta}$ is contained in some true response $\bm Z_k(X_k)$. In
other words, the model assumes that no information about the linear
classifier is lost if we focus on inferring the true response
instead. The last assumption is focused on identifiability of the
model: since we infer by observing a signal, it is critical that we
can tell the conditional distributions of these signals apart, and the
latter condition guarantees this.

One of the benefits this noise model provides is allowing us to easily
update our beliefs of $\bm{\theta}$. For a given question $X \in \X^m$
and true response $z\in\Z$, let
\begin{equation}
  \label{eq:characteristicpolytopes}
  A^{(z)}(X) 
  = \left\{\theta \in \Theta:\, \begin{array}{l l}\theta^T x^{(z)} \geq \theta^T x^{(i)} \quad
  & \forall i > z \\ \theta^T x^{(z)} > \theta^T x^{(i)} & \forall i < z \end{array}\right\}.
\end{equation}
These $m$ sets form a partition of $\Theta$ that depend on the
question $X$ we ask at each time epoch, where each set $A^{(z)}$
corresponds to all linear classifiers $\bm{\theta}$ that are consistent
with the true response $\bm Z=z$.

Let $\mu_k$ denote the prior measure of $\bm{\theta}$ at time epoch
$k$. Throughout the paper, we assume that $\mu_k$ is absolutely
continuous with respect to $d$-dimensional Lebesgue measure, admitting
a corresponding Lebesgue density $p_k$. At every epoch, we ask the
user a comparative question that asks for the most preferred option in
$\bm{X}_k =\{x_1,x_2,\dots,x_m\}$. We observe signal $\bm Y_k(X_k)$,
and accordingly update the prior.
\begin{lemma}
  Suppose that the \ref{assm:noise} hold. Then we can write the posterior
  $p_{k+1}$ as
  \begin{equation}
    \label{eq:posterior}
    p_{k+1}\left(\theta\,|\,\bm Y_k(X_k)=y\right) = 
    \left(\frac{\sum_{z\in\Z}\mathbb{I}(\theta\in A^{(z)}(X_k))\;f^{(z)}(y)}
         {\sum_{z\in\Z} \mu_k\left(A^{(z)}(X_k)\right)\,f^{(z)}(y)}\right)\, p_k(\theta),
  \end{equation}
  where $\mathbb{I}$ denotes the indicator function.
\end{lemma}
\begin{proof}
  Using Bayes' rule, we see
  \begin{align*}
    p_{k+1}(\theta\,|\,\bm Y_k(X_k) = y)
    &\propto \mathbb{P}_k(\bm Y_k(X_k)=y\,|\,\bm{\theta} = \theta)\cdot p_k(\theta) \\
    &= \sum_{z\in\Z}
    \mathbb{P}_k(\bm Y_k(X_k)=y\,|\,\bm Z_k(X_k)=z,\,\bm{\theta} =
    \theta)\cdot
    \mathbb{P}_k(\bm Z_k(X_k)=z\,|\,\bm{\theta} = \theta) \cdot p_k(\theta). \\
    \intertext{Now we use a property of $\bm Y_k(X_k)$ and $\bm Z_k(X_k)$ from the
      \ref{assm:noise}, namely that $\bm Y_k(X_k)$ and $\bm{\theta}$ are conditionally
      independent given $\bm Z_k(X_k)$. This implies}
    p_{k+1}(\theta\,|\,\bm Y_k(X_k)=y)
    &\propto \sum_{z\in\Z}\mathbb{P}_k(\bm Y_k(X_k)=y\,|\,\bm Z_k(X_k)=z)\cdot
    \mathbb{P}_k(\bm Z_k(X_k)=z\,|\,\bm{\theta} = \theta) \cdot p_k(\theta) \\
    &= \sum_{z\in\Z} f^{(z)}(y) \cdot \mathbb{I}\left(\theta \in A^{(z)}(X_k)\right) \cdot
      p_k(\theta),
  \end{align*}
  where the last line is true because $\bm Z_k(X_k)$ is a
  deterministic function of $\bm{\theta}$ and $X_k$. Normalizing
  to ensure the density integrates to one gives the result.
\end{proof}
The \ref{assm:noise} allow us to easily update the prior on
$\bm{\theta}$. As we will see next, they also allow us to easily express
the conditions required to maximize one-step entropy reduction.

\section{Posterior Entropy}
\label{sec:entropy}

We focus on how we select the alternatives we offer to the
user. First, we need to choose a metric to evaluate the effectiveness
of each question. One option is to use a measure of dispersion of the
posterior distribution of $\bm\theta$, and the objective is to
decrease the amount of spread as much as possible with every
question. Along these lines, we elect to use differential entropy for
its tractability.

For a probability density $p$, the differential entropy of $p$ is
defined as
\begin{equation*}
  H(p) = \int_\Theta -p(\theta) \log_2 p(\theta) \,d\theta.
\end{equation*}
For the entirety of this paper, all logarithms are base-2, implying
that both Shannon and differential entropy are measured in
bits. Because we ask the user multiple questions, it is important to
incorporate the previous response history $\Yhist{k}$ when considering
posterior entropy. Let $H_k$ be the entropy operator at time epoch $k$
such that $H_k(\bm\theta) = H(\bm\theta\,|\,\Yhist{k})$, which takes
into account all of the previous observation history
$\Yhist{k}$. Occasionally, when looking at the performance of a policy
$\pi$, we would want to randomize over all such histories. This is
equivalent to the concept of \textit{conditional entropy}, with
$H^\pi(\bm\theta\,|\,\bmYhist{k}) = \mathbb{E}^\pi\left[
  H_k(\bm\theta) \right]$.

Throughout the paper, we represent discrete distributions as
vectors. Accordingly, define $\Delta^m = \{u \in \R^m:\, \sum_z u^{(z)} =
1 ,\,u\geq 0\}$ to be the set of discrete probability distributions
over $m$ alternatives. For a probability distribution $u \in
\Delta^m$, we define $h(u)$ to be the Shannon entropy of that discrete
distribution, namely
\begin{equation*}
  h(u) = \sum_{z\in\Z} - u^{(z)} \log_2 u^{(z)}.
\end{equation*}
Here, we consider discrete probability distributions over the
alternatives we offer, which is why distributions $u$ are indexed by
$z \in \Z$.

Since stochastic matrices are be used to model some noise channels,
we develop similar notation for matrices. Let $\Delta^{m\times
  m}$ denote the set of $m\times m$ row-stochastic matrices. Similarly
to how we defined the Shannon entropy of a vector, we define $h(P)$ as
an $m$-vector with the Shannon entropies of the rows of $P$ as its
components. In other words,
\begin{equation*}
  h\left(P\right)^{(z)} = \sum_{y\in\Y} -P^{(zy)}\,\log_2 P^{(zy)}.
\end{equation*}

An important concept in information theory is \textit{mutual
  information}, which measures the entropy reduction of a random
variable when conditioning on another. It is natural to ask about the
relationship between the information gain of $\bm{\theta}$ and that of
$\bm Z_k(X_k)$ after observing signal $\bm Y_k(X_k)$. Mutual
information in this context is defined as
\begin{equation}
  I_k(\bm\theta;\bm Y_k(X_k)) = H_k(\bm\theta)-H_k(\bm\theta\,|\,\bm
  Y_k(X_k)).  
\end{equation}
One critical property of mutual information is that it is symmetric,
or in other words, $I_k(\bm\theta;\bm Y_k(X_k)) = I_k(\bm
Y_k(X_k);\bm\theta)$ \citep[see][p.~20]{cover1991}. In the context of
our model, this means that observing signal $\bm Y_k(X_k)$ gives us
the same amount of information about linear classifier $\bm\theta$ as
would observing the linear classifier would provide about the
signal. This is one property we exploit throughout the paper, since
the latter case only depends on the noise channel, which by assumption
does not change over time. We show in
Theorem~\ref{thm:entropyidentities} below that the \ref{assm:noise}
allow us to determine how the noise channel affects the posterior
entropy of linear classifier $\bm\theta$.

The first identity, given by \eqref{eq:entropyidentity1}, says that
the noise provides an additive effect with respect to entropy,
particularly because the noise does not depend on $\bm\theta$
itself. The second identity, given by \eqref{eq:entropyidentity2},
highlights the fact that $\bm Y_k(X_k)$ provides the same amount of
information on the linear classifier $\bm\theta$ as it does on the
true answer $\bm Z_k(X_k)$ for a given question. This means that the
entropy of both $\bm\theta$ and $\bm Z_k(X_k)$ are reduced by the same
number of bits when asking question $X_k$. Intuitively, asking the
question that would gain the most clarity from a response would also
do the same for the underlying linear classifier. This is formalized
in Theorem~\ref{thm:entropyidentities} below.

\begin{theorem}
  \label{thm:entropyidentities}
  The following information identities hold under the \ref{assm:noise}
  for all time epochs $k$. The first is the \textbf{Noise Separation
    Equality}, namely
  \begin{equation}
    \label{eq:entropyidentity1}
    H_k(\bm{\theta}\,|\,Y_k(X_k))
    = H_k(\bm{\theta}\,|\,\bm Z_k(X_k))
    + H_k(\bm Z_k(X_k)\,|\,\bm Y_k(X_k)),
  \end{equation}
  and the \textbf{Noise Channel Information Equality}, given by
  \begin{equation}
    \label{eq:entropyidentity2}
    I_k(\bm{\theta};\bm Y_k(X_k)) = I(\bm Z_k(X_k);\bm Y_k(X_k)),
  \end{equation}
  where the latter term does not depend on response history
  $\Yhist{k}$.
\end{theorem}

\begin{proof}
  Using the symmetry of mutual information,
  \begin{equation*}
    H_k(\bm\theta\,|\bm Y_k(X_k)) - H_k(\bm\theta\,|\,\bm
    Y_k(X_k),\bm Z_k(X_k)) = H_k(\bm Z_k(X_k)\,|\,\bm Y_k(X_k))
    - H(\bm Z_k(X_k)\,|\,\bm \theta,\bm Y_k(X_k)).
  \end{equation*}
  Further, we know $H_k(\bm\theta\,|\,\bm Y_k(X_k),\bm Z_k(X_k)) =
  H_k(\bm\theta\,|\,\bm Z_k(X_k))$ because $\bm Y_k(X_k)$ and
  $\bm\theta$ are conditionally independent given $\bm
  Z_k(X_k)$. Also, since $\bm Z_k(X_k)$ is a function of $\bm\theta$
  and $X_k$, it must be that $H_k(\bm Z_k(X_k)\,|\,\bm\theta,\bm
  Y_k(X_k)) = 0$. Putting these together gives us the first
  identity. To prove the second identity, we use the fact that
  \begin{equation*}
    H_k(\bm\theta\,|\,\bm Z_k(X_k)) + H_k(\bm Z_k(X_k))
    = H_k(\bm Z_k(X_k)\,|\,\bm\theta) + H_k(\bm\theta).
  \end{equation*}
  Again, $H_k(\bm Z_k(X_k)\,|\,\bm\theta) = 0$ because $\bm Z_k(X_k)$
  is a function of $\bm\theta$ and $X_k$. This yields
  $H_k(\bm\theta\,|\,\bm Z_k(X_k)) = H_k(\bm\theta) - H_k(\bm
  Z_k(X_k))$. Substitution into the first identity gives us
  \begin{equation*}
    H_k(\bm\theta) - H_k(\bm\theta\,|\,\bm Y_k(X_k))
    = H_k(\bm Z_k(X_k)) - H_k(\bm Z_k(X_k)\,|\,\bm Y_k(X_k)),
  \end{equation*}
  which is \eqref{eq:entropyidentity2}, by definition of mutual
  information. Finally, by the \ref{assm:noise}, signal $\bm Y_k(X_k)$
  is conditionally independent of history $\Yhist{k}$ given $\bm
  Z_k(X_k)$, and therefore, $I_k(\bm Z_k(X_k);\bm Y_k(X_k)) = I(\bm
  Z_k(X_k);\bm Y_k(X_k))$.
\end{proof}

The entropy pursuit policy is one that maximizes the reduction in
entropy of the linear classifier, namely $I_k(\bm\theta;\bm Y_k(X_k))
= H_k(\bm\theta) - H_k(\bm\theta\,|\,\bm Y_k(X_k))$, at each time
epoch. We leverage the results from
Theorem~\ref{thm:entropyidentities} to find conditions on questions
that maximally reduce entropy in the linear classifier
$\bm\theta$. However, we first need to introduce some more notation.

For a noise channel parameterized by $f = \{f^{(z)}:\,z\in\Z\}$, let
$\varphi$ denote the function on domain $\Delta^m$ defined as
\begin{equation}
  \label{eq:genchanneleq}
  \varphi(u\,;f) = H\left( \sum_{z \in \Z} u^{(z)} f^{(z)} \right) - \sum_{z \in
    \Z} u^{(z)} \, H\left( f^{(z)} \right).
\end{equation}
We will show in Theorem~\ref{thm:entropyPursuit} that
\eqref{eq:genchanneleq} refers to the reduction in entropy from asking
a question, where the argument $u\in\Delta^m$ depends on the
question. We define the channel capacity over noise channel $f$,
denoted $C(f)$, to be the supremum of $\varphi$ over this domain,
namely
\begin{equation}
  \label{eq:channelcapacity}
  C(f) = \sup_{u \in \Delta^m} \varphi(u\,;f),
\end{equation}
and this denotes the maximal amount of entropy reduction at every
step. These can be similarly defined for a discrete noise channel.
For a noise channel parameterized by transmission matrix $P$, we
define
\begin{equation}
  \label{eq:channeleq}
  \varphi(u\,;\,P) = h(u^T P) - u^T h(P),
\end{equation}
and $C(P)$ is correspondingly the supremum of $\varphi(\cdot\,;\,P)$
in its first argument. In Theorem~\ref{thm:entropyPursuit} below, we
show that $\varphi(u\,;\,f)$ is precisely the amount of entropy over
linear classifiers $\bm\theta$ reduced by asking a question with
respective predictive distribution $u$ under noise channel $f$.
\begin{theorem}
  \label{thm:entropyPursuit}
  For a given question $X\in\X^m$, define $u_k(X) \in \Delta^m$ such
  that $u^{(z)}_k(X) = \mu_k\left(A^{(z)}(X)\right)$ for all $z\in\Z$. Suppose
  that the \ref{assm:noise} hold. Then for a fixed noise channel
  parameterized by $f = \{f^{(z)}:\,z\in\Z\}$,
  \begin{equation}
    I_k(\bm\theta;\bm Y_k(X_k)) = \varphi\left(u_k(X_k)\,;\,f\right).
  \end{equation}
  Consequently, for all time epochs $k$, we have
  \begin{equation}
    \label{eq:channelcapacitybound}
    \sup_{X_k \in \X^m} I_k(\bm\theta;\bm Y_k(X_k)) \leq C(f),
  \end{equation}
  and there exists $u_* \in \Delta^m$ that attains the
  supremum. Moreover, if there exists some $X_k \in \X^m$ such that
  $u_k(X_k) = u_*$, then the upper bound is attained.
\end{theorem}
\begin{proof}
  We first use \eqref{eq:entropyidentity2} from
  Theorem~\ref{thm:entropyidentities}, namely that $I_k(\bm\theta;\bm
  Y_k(X_k)) = I_k(\bm Z_k(X_k);\bm Y_k(X_k))$. We use the fact that
  mutual information is symmetric, meaning that the entropy reduction
  in $\bm Z_k(X_k)$ while observing $\bm Y_k(X_k)$ is equal to that in
  $\bm Y_k(X_k)$ while observing $\bm Z_k(X_k)$. Putting this together
  with the definition of mutual information yields
  \begin{align*}
    I_k(\bm\theta;\bm Y_k(X_k))
    &= I_k(\bm Z_k(X_k);\bm Y_k(X_k)) \\
    &= H_k(\bm Y_k(X_k)) - H_k(\bm Y_k(X_k)\,|\,\bm Z_k(X_k)) \\
    &= H\left(\sum_{z\in\Z} \mathbb{P}_k(\bm Z_k(X_k)=z)\,f^{(z)}\right)
    - \sum_{z\in\Z} \mathbb{P}_k(\bm Z_k(X_k)=z)\, H(f^{(z)}) \\
    &= H\left(\sum_{z\in\Z} \mu_k\left(A^{(z)}(X_k)\right)\,f^{(z)}\right)
    - \sum_{z\in\Z} \mu_k\left(A^{(z)}(X_k)\right)\, H(f^{(z)}),
  \end{align*}
  which is equal to $\varphi(u_k(X_k)\,;f)$, where
  $u^{(z)}_k(X_k) = \mu_k\left(A^{(z)}(X_k)\right)$. Therefore,
  the optimization problem in \eqref{eq:channelcapacitybound} is
  equivalent to
  \begin{equation*}
    \sup_{X_k\in\X^m} \varphi\left(u_k(X_k)\,;f\right).
  \end{equation*}
  Since $\{u_k(X):\,X\in\X^m\} \subseteq \Delta^m$, we can
  relax the above problem to
  \begin{equation*}
    \sup_{u\in\Delta^m} \varphi(u\,;f). 
  \end{equation*}
  It is known that mutual information is concave in its probability
  mass function \citep[see][p.~31]{cover1991}, and strictly concave
  when the likelihood functions $f^{(z)}$ differ on a set of positive
  measure. Thus, for a fixed noise channel $f$, $\varphi(\cdot\,;\,f)$
  is concave on $\Delta^m$, a compact convex set, implying an optimal
  solution $u_*$ exists and the optimal objective value $C(f)>0$ is
  attained. Further, if we can construct some $X_k\in\X^m$ such that
  $\mu_k\left(A^{(z)}(X_k)\right) = u^{(z)}_*$ for every $z\in\Z$,
  then the upper bound is attained.
\end{proof}
We have shown that entropy reduction of the posterior of $\bm\theta$
depends only on the implied predictive distribution of a given
question and structure of the noise channel. If we are free to
fabricate alternatives to achieve the optimal predictive distribution,
then we reduce the entropy of the posterior by a fixed amount $C(f)$ at
every time epoch. Perhaps the most surprising aspect of this result is
the fact that the history $\Yhist{k}$ plays no role in the
\textit{amount} of entropy reduction, which is important for showing
that entropy pursuit is an optimal policy for reducing entropy over
several questions.

In practice, one can usually ask more than one question, and it is
natural to ask if there is an extension that gives us a bound on the
posterior entropy after asking several questions. Using the results in
Theorem~\ref{thm:entropyPursuit}, we can derive an analogous lower
bound for this case.
\begin{corollary}
  \label{cor:kstepentropy}
  For a given policy $\pi\in\Pi$, we can write the entropy of
  linear classifier $\bm\theta$ after $K$ time epochs as
  \begin{equation}
    H(\bm\theta) - H^\pi\left(\bm\theta\,|\,\bmYhist{K}\right) =
    \mathbb{E}^\pi \left[ \sum_{k=1}^K \varphi(u_k(X_k)\,;\,f) \right],
  \end{equation}
  and a lower bound for the differential entropy of $\bm \theta$ after
  asking $K$ questions is given below by
  \begin{equation}
    \inf_{\pi\in\Pi} H^\pi(\bm\theta\,|\,\bmYhist{K})
    \geq H(\bm\theta) - K\cdot C(f).
  \end{equation}
  Further, if for a given policy $\pi$ and history $\Yhist{k}$
  indicates that comparative question $X_k$ should be posed to the
  user, then the lower bound is attained if and only if $u_k(X_k) =
  u_*$, with $u_*$ as defined in
  Theorem~\ref{thm:entropyPursuit}. Thus, entropy pursuit is an
  optimal policy.
\end{corollary}
\begin{proof}
  Using the information chain rule, we can write the entropy reduction
  for a generic policy $\pi\in\Pi$ as
  \begin{align*}
    H(\bm\theta) - H^\pi(\bm\theta\,|\,\bmYhist{K})
    &= I^\pi\left(\bm\theta;\bmYhist{K}\right) \\
    &= \sum_{k=1}^K \mathbb{E}^\pi \left[\vphantom{\frac10} I_k\left(\bm\theta;\bm
      Y_k(X_k)\right) \right]
    \leq K\cdot C(f),
  \end{align*}
  where the last inequality comes directly from
  Theorem~\ref{thm:entropyPursuit}, and the upper bound is attained if
  and only if $u_k(X_k) = u_*$ for every $k=1,2,\dots,K$. This
  coincides with the entropy pursuit policy. 
\end{proof}
Essentially, Corollary~\ref{cor:kstepentropy} shows that the greedy
entropy reduction policy is, in fact, the optimal policy over any time
horizon. However, there is still an important element that is missing:
how can we ensure that there exists some alternative that satisfies
the entropy pursuit criteria? We address this important concern in
Section~\ref{subsec:continuum}.


\subsection{Optimality Conditions for Predictive Distribution}
\label{subsec:optimalitycond}

Because of the properties of entropy, the noise channel function
$\varphi$ has a lot of structure. We use this structure to find
conditions for a non-degenerate optimal predictive distribution $u_*$
as well as derive sensitivity results that allow the optimality gap of
a close-to-optimal predictive distribution to be estimated.

Before we prove structural results for the channel equation $\varphi$,
some more information theoretic notation should be introduced. Given two
densities $f^{(i)}$ and $f^{(j)}$, the \textit{cross entropy} of these
two densities is defined as
\begin{equation*}
  H\left(f^{(i)},f^{(j)}\right) = \int_{\Y} -f^{(i)}(y) \log_2
  f^{(j)}(y)\,dy.
\end{equation*}
Using the definition of cross entropy, the Kullback-Leibler divergence
between two densities $f^{(i)}$ and $f^{(j)}$ is defined as
\begin{equation*}
  \KL\left( f^{(i)}\,\middle\|\,f^{(j)} \right) = H(f^{(i)},f^{(j)}) -
  H(f^{(i)}).
\end{equation*}
Kullback-Leibler divergence is a tractable way of measuring the
difference of two densities. An interesting property of
Kullback-Leibler divergence is that for any densities $f^{(i)}$ and
$f^{(j)}$, $\KL(f^{(i)}\|f^{(j)}) \geq 0$, with equality if and only
if $f^{(i)} = f^{(j)}$ almost surely. Kullback-Leibler divergence plays
a crucial role the first-order information for the channel equation
$\varphi$.

We now derive results that express the gradient and Hessian of
$\varphi$ in terms of the noise channel, which can either be
parameterized by $f$ in the case of a density, or by a fixed
transmission matrix $P$ in the discrete noise channel case. For these
results to hold, we require the cross entropy $H(f^{(i)},f^{(j)})$ to
be bounded in magnitude for all $i,j\in\Z$, which is an entirely
reasonable assumption.

\begin{lemma}
  \label{lem:channelgrad}
  For a fixed noise channel characterized by $f =
  \{f^{(z)}:\,z\in\Z\}$, if the cross entropy terms
  $H(f^{(i)},f^{(j)})$ are bounded for all $i,j\in\Z$, then the first
  and second partial derivatives of $\varphi$ with respect to $u$ are
  given by
  \begin{align*}
    \frac{\partial \varphi(u\,;\,f)}{\partial u^{(z)}}
    &= \KL\left( f^{(z)} \;\middle\|\; \sum_{i\in\Z} u^{(i)} f^{(i)}\right)  - \xi \\
    \frac{\partial^2 \varphi(u\,;\,f)}{\partial u^{(z)}\,\partial u^{(w)}}
    &= -\xi\int_{\Y} \frac{f^{(z)}(y)\,f^{(w)}(y)}{\sum_{i\in\Z} u^{(i)} f^{(i)}(y)}\,dy,
  \end{align*}
  where $\xi = \log_2 e$, and $\KL(\cdot\,\|\,\cdot)$ is the
  Kullback-Leibler Divergence.

  In particular, if a discrete noise channel is parameterized by
  transmission matrix $P$, the gradient and Hessian matrix of
  $\varphi$ can be respectively expressed as
  \begin{align*}
    \nabla_u\,\varphi(u\,;\,P)
    &= -P\log_2\left( P^T u \right)- h(P)  - \xi e  \\
    \nabla_u^2\, \varphi(u\,;\,P)
    &= -\xi\, P \left(\diag\left(u^T P\right) \right)^{-1} P^T,
    \end{align*}
  where the logarithm is taken component-wise.
\end{lemma}
\begin{proof}
  We first prove the result in the more general case when the noise
  channel is parameterized by $f$. From the definition of $\varphi$,
  \begin{equation*}
    \varphi(u\,;\,f) = \int_{\Y} -\left(\sum_{i\in\Z} u^{(i)} f^{(i)}(y) \right)
    \log_2\left(\sum_{i\in\Z} u^{(i)} f^{(i)}(y) \right)\,dy - \sum_{i\in\Z}
    u^{(i)} H(f^{(i)}).
  \end{equation*}
  Since $t \mapsto -\log t$ is convex, by Jensen's inequality,
  $H(f^{(z)},\sum_i u^{(i)} f^{(i)}) \leq \sum_i u^{(i)}
  H(f^{(z)},f^{(i)})$, which is bounded. By the Dominated Convergence
  Theorem, we can switch differentiation and integration operators,
  and thus,
  \begin{align*}
    \frac{\partial}{\partial u^{(z)}}\,\varphi(u\,;\,f)
    &= \int_{\Y} -f^{(z)}(y)\log_2\left(\sum_{i\in\Z} u^{(i)}
    f^{(i)}(y)\right)\,dy - \xi - H(f^{(z)}) \\
    &= \KL\left(f^{(z)}\;\middle\|\;\sum_{i\in\Z} u^{(i)} f^{(i)} \right) - \xi.
  \end{align*}
  Concerning the second partial derivative, Kullback-Leibler
  divergence is always non-negative, and therefore, Monotone
  Convergence Theorem again allows us to switch integration and
  differentiation, yielding
  \begin{equation*}
    \frac{\partial^2\varphi(u\,;\,f)}{\partial u^{(z)}\,\partial u^{(w)}}
    = - \xi\, \int_{\Y} \frac{f^{(z)}(y)\,f^{(w)}(y)}{\sum_{i\in\Z} u^{(i)} f^{(i)}(y)} \,dy.
  \end{equation*}
  For the discrete noise channel case, the proof is analogous to
  above, using Equation~\eqref{eq:channeleq}. Vectorizing yields
  \begin{align*}
    \nabla_u\,\varphi(u\,;P) &= -P \left( \log_2(P^T u) + \xi e\right) - h(P)
    \\
    &= -P\log_2\left(P^T u\right) - h(P) - \xi e.
  \end{align*}
  Similarly, the discrete noise channel analogue for the second
  derivative is
  \begin{equation*}
    \frac{\partial^2\varphi(u\,;P)}{\partial u^{(z)} \partial u^{(w)}} =
    - \xi\,\sum_{y\in\Y} \frac{P^{(zy)} P^{(wy)}}{\sum_{i\in\Z} u^{(i)} P^{(iy)}},
  \end{equation*}
  and vectorizing gives us the Hessian matrix.
\end{proof}

One can now use the results in Lemma~\ref{lem:channelgrad} to find
conditions for an optimal predictive distribution for a noise channel
parameterized either by densities $f = \{f^{(z)}:\,z\in\Z\}$ or
transmission matrix $P$. There has been much research on how to find
the optimal predictive distribution $u_*$ given a noise channel, as in
\citet{gallager1968information}. Generally, there are two methods for
finding this quantity. The first relies on solving a constrained
concave maximization problem by using a first-order method. The other
involves using the Karush-Kuhn-Tucker conditions necessary for an
optimal solution \citep[see][p.~91 for
  proof]{gallager1968information}.
\begin{theorem}[Gallager]
  \label{thm:kktsolution}
  Given a noise channel parameterized by $f = \{f^{(z)}:\,z\in\Z\}$,
  the optimal predictive distribution $u_*$ satisfies
  \begin{equation*}
    \KL\left( f^{(z)}\,\middle\|\,\sum_{i\in\Z}u^{(i)}f^{(i)} \right)
    \begin{cases} = C(f) & u_*^{(z)} > 0 \\
      < C(f) & u_*^{(z)} = 0,\\
    \end{cases}
  \end{equation*}
  where $C(f)$ is the channel capacity.
\end{theorem}
The difficulty in solving this problem comes from determining whether
or not $u_*^{(z)} > 0$. In the context of preference elicitation, when
fixing the number of offered alternatives $m$, it is critical for
every alternative to contribute to reducing uncertainty. However,
having a noise channel where $u_*^{(z)} = 0$ implies that it is more
efficient to learn without offering alternative $z$.

To be specific, we say that a noise channel parameterized by $f =
\{f^{(z)}:\,z\in\Z\}$ is \textit{admissible} if there exists some $f_*
\in \Int\left(\Hull(f)\right)$ such that for all $z\in\Z$,
\begin{equation*}
  \KL\left(f^{(z)} \,\middle\|\, f_*\right) = C
\end{equation*}
for some $C > 0$. Otherwise, we say the noise channel is
\textit{inadmissible}. Admissibility is equivalent the existence of a
predictive distribution $u_* > 0$ where all $m$ alternatives are used
to learn a user's preferences. For pairwise comparisons, any noise
channel where $f^{(1)}$ and $f^{(2)}$ differ on a set of non-zero
Lebesgue measure is admissible. Otherwise, for $m>2$, there are
situations when $u_*^{(z)} = 0$ for some $z\in\Z$, and
Lemma~\ref{lem:zeropredictive} provides one of them. In particular, if
one density $f^{(z)}$ is a convex combination of any of the others,
then the optimal predictive distribution will always have $u_*^{(z)} =
0$.

\begin{lemma}
  \label{lem:zeropredictive}
  Suppose the noise channel is parameterized by densities $f =
  \{f^{(z)}:\,z\in\Z\}$, and its corresponding optimal predictive
  distribution is $u_*$. If there
  exists $\lambda^{(i)} \geq 0$ for $i\neq z$ such that $\sum_{i\neq
    z} \lambda^{(i)} = 1$ and $f^{(z)}(y) = \sum_{i\neq z}
  \lambda^{(i)} f^{(i)}(y)$ for all $y\in\Y$, then $u_*^{(z)} = 0$.
\end{lemma}
\begin{proof}
  Suppose $f^{(z)} = \sum_{i\neq z} \lambda^{(i)}f^{(i)}$. Take
  any $u\in\Delta^m$ such that $u^{(z)} > 0$. We will construct a
  $\bar u \in \Delta^m$ such that $\bar u ^{(z)} = 0$ and
  $\varphi(\bar u\,;\,f) > \varphi(u\,;\,f)$. Define $\bar u$ as
  \begin{equation*}
    \bar u^{(i)} = \begin{cases}
      u^{(i)} + \lambda^{(i)} u^{(z)} & i\neq z \\
      0 & i = z. \end{cases}
  \end{equation*}
  It is easy to verify that $\sum_i \bar u^{(i)} f^{(i)} = \sum_i
  u^{(i)} f^{(i)}$. But since entropy is strictly concave, we have
  $H\left(f^{(z)}\right) > \sum_{i\neq z}\lambda^{(i)}
  f^{(i)}$. Consequently,
  \begin{align*}
    \varphi(u\,;\,f) &= H\left( \sum_{i\in\Z} u^{(i)} f^{(i)}\right) -
    \sum_{i\in\Z} u^{(i)} H\left(f^{(i)}\right) \\
    &= H\left( \sum_{i\neq z} \bar u^{(i)} f^{(i)} \right)
    - \sum_{i\neq z} u^{(i)} H(f^{(i)}) - u^{(z)} H(f^{(z)}) \\
    &< H\left( \sum_{i\neq z} \bar u^{(i)} f^{(i)} \right)
    - \sum_{i\neq z} u^{(i)} H(f^{(i)}) - u^{(z)} \sum_{i\neq z}
    \lambda^{(i)} H(f^{(i)}) \\
    &= H\left(\sum_{i\neq z} \bar u^{(i)} f^{(i)} \right)
    - \sum_{i\neq z} \bar u^{(i)} H(f^{(i)}) = \varphi(\bar u\,;\,f),
  \end{align*}
  and therefore, one can always increase the objective value of
  $\varphi$ by setting $u^{(z)} = 0$.
\end{proof}
Of course, there are other cases where the predictive distribution
$u_*$ is not strictly positive for every $z\in\Z$. For example, even
if one of the densities is an \textit{approximate} convex combination,
the optimal predictive distribution would likely still have $u_*^{(z)}
= 0$. In general, there is no easy condition to check whether or not
$u_* > 0$. However, our problem assumes $m$ is relatively small, and
so it is simpler to find $u_*$ and confirm the channel is
admissible. In the case of a discrete noise channel,
\citet{shannon1948mathematical} gave an efficient way to do this by
solving a relaxed version of the concave maximization problem,
provided that the transmission matrix $P$ is invertible.

\begin{theorem}[Shannon]
  \label{thm:firstorder}
  For a discrete noise channel parameterized by a non-singular
  transmission matrix $P$, let
  \begin{equation}
    \label{eq:postpredictive}
    v = \frac{\exp\left(-\xi^{-1} P^{-1} h(P)\right)}{e^T\,
      \exp\left(-\xi^{-1} P^{-1}h(P)\right)},
  \end{equation}
  where the exponential is taken component-wise. If there exists $u>0$
  such that $u^T P = v^T$, then $u \in \Int(\Delta^m)$ is the optimal
  predictive distribution, meaning that
  $\nabla_u\,\varphi(u\,;P) = \beta e$ for some $\beta\in\mathbb{R}$,
  and $\varphi(u_*\,;P) = C(P)$, and the noise channel is
  admissible. Otherwise, then there exists some $z\in\Z$ such that
  $u^{(z)} = 0$, and the noise channel is inadmissible.
\end{theorem}
\begin{proof}
  Using \eqref{eq:channeleq} and Lagrangian relaxation,
  \begin{align*}
    \sup_{u:\,e^T u = 1} \varphi(u\,;P)
    &= \sup_{u:\,e^T u = 1} h(u^T P) - u^T h(P) \\
    &= \sup_{u \in\mathbb{R}^m} \inf_{\lambda \in \mathbb{R} \vphantom{R^d}}
    h(u^T P) - u^T h(P) - \lambda \left( e^T u - 1\right).
  \end{align*}
  Differentiating with respect to $u$ and setting equal to zero
  yields
  \begin{equation*}
    -P \log_2 \left(P^T u\right) - h(P) + \xi e - \lambda e = 0,
  \end{equation*}
  and since $P$ is invertible,
  \begin{equation*}
    -\log_2 \left(P^T u\right) = P^{-1} h(P) + (\lambda - \xi) e,
  \end{equation*}
  since $Pe = e$ for all stochastic matrices $P$. Algebra yields
  \begin{align*}
    P^T u &= \exp\left(-\xi^{-1} P^{-1}h(P)
    + (\lambda/\xi - 1)e\right) \\
    &= \Lambda\cdot \exp\left(-\xi^{-1} P^{-1} h(P)\right),
  \end{align*}
  where $\Lambda = \exp\left(\lambda/\xi - 1\right)$ is some positive
  constant. We require $e^T u = 1$, and if $u^T P = v^T$, it must be
  that
  \begin{equation*}
    u^T e = u^T P e = v^T e,
  \end{equation*}
  implying that $e^T u = 1$ if and only if $e^T v = 1$. Hence,
  $\Lambda$ is a normalizing constant that allows $v^T P = 1$. Thus,
  we can set $v$ as in \eqref{eq:postpredictive}, and now it is clear
  that $v \in \Delta^m$. We can invert $P$ to find an explicit form
  for $u$, but $P^{-T} v$ is only feasible for the original
  optimization problem if it is non-negative. However, if there exists
  some $u\in\Delta^m$ such that $u^T P = v^T$, then the optimal
  solution to the relaxed problem is feasible for the original
  optimization problem, proving the theorem.
\end{proof}

If there does not exist some $u\geq 0$ that satisfied $u^T P = v^T$
for $v$ defined in \eqref{eq:postpredictive}, then the non-negativity
constraint would be tight, and $u^{(z)}_* = 0$ for some $z\in\Z$. In
this case, the noise channel is inadmissible, because it implies
asking the optimal question under entropy pursuit would assign zero
probability to one of the alternatives being the model consistent
answer, and thus posits a question of strictly less than $m$
alternatives to the user.

The condition of $P$ being non-singular has an enlightening
interpretation. Having a non-singular transmission matrix implies
there would be no two distinct predictive distributions for $\bm
Z_k(X_k)$ that yield the same predictive distribution over $\bm
Y_k(X_k)$. This is critical for the model to be identifiable, and
prevents the previous problem of having one row of $P$ being a convex
combination of other rows. The non-singular condition is reasonable in practice:
it is easy to verify that matrices in the form $P = \alpha I +
(1-\alpha)ve^T$ for some $v\in\Delta^m$ is invertible if and only if
$\alpha > 0$. Transmission matrices of this type are fairly
reasonable: with probability $\alpha$, the user selects the ``true
response,'' and with probability $(1-\alpha)$, the user selects from
discrete distribution $v$, regardless of $\bm Z_k(X_k)$. The symmetric
noise channel is a special case of this. In general, if one models $P
= \alpha I + (1-\alpha) S$, where $S$ is an $m\times m$ stochastic
matrix, then $P$ is non-singular if and only if $-\alpha/(1-\alpha)$
is not an eigenvalue of $S$, which guarantees that $P$ is invertible when
$\alpha > 1/2$. Nevertheless, regardless of whether or not $P$ is
singular, it is relatively easy to check the admissibility of a noise
channel, and consequently conclude whether or not it is a good
modeling choice for the purpose of preference elicitation.

\subsection{Sensitivity Analysis}
\label{subsec:finite alts}

In reality, we cannot always fabricate alternatives so that the
predictive distribution is exactly optimal. In many instances, the set
of alternatives $\X$ is finite. This prevents us from choosing an
$X_k$ such that $u_k(X_k) = u_*$ exactly. But if we can find a
question that has a predictive distribution that is sufficiently close
to optimal, then we can reduce the entropy at a rate that is close to
the channel capacity. Below, we elaborate on our definition of
sufficiently close by showing $\varphi$ is strongly concave, using the
Hessian to construct quadratic upper and lower bounds on the objective
function $\varphi$.

\begin{theorem}
  \label{thm:taylorbounds}
  If there exists $u_*\in\Delta^m$ such that $u_*>0$ and
  $\varphi(u_*\,;\,f) = C(f)$ (i.e., if the noise channel is
  admissible), then there exist constants $0 \leq r(f) \leq R(f)$ such
  that
  \begin{equation*}
    r(f)\cdot\|u-u_*\|^2 \leq C(f) - \varphi(u\,;\,f)
    \leq R(f)\cdot \|u-u_*\|^2.
  \end{equation*}
  Further, suppose transmission matrix $P$ encoding a discrete noise
  channel is non-singular, and has minimum probability $\kappa_1 =
  \min_{zy} P^{(zy)} > 0$, maximum probability $\kappa_2 = \max_{zy}
  P^{(zy)}$, channel capacity $C(P)$ and distribution $u_*$ such that
  $\varphi(u_*\,;P) = C(P)$. If $u_* > 0$, we have
  \begin{equation*}
    \frac{\xi}{2\kappa_2}\left\|(u-u_*)^T P\right\|^2 \leq C(P) - \varphi(u\,;P) \leq
    \frac{\xi}{2\kappa_1}\left\|(u-u_*)^T P\right\|^2
  \end{equation*}
  for all $u \in \Delta^m$, with $\xi = \log_2 e$.
\end{theorem}
\begin{proof}
  The $(z,w)$ component of $-\nabla^2 \varphi(\cdot\,;\,f)$ is lower
  bounded by
  \begin{equation}
    \label{eq:firstmatrixbound} 
    \int_{\Y} \frac{f^{(z)}(y)f^{(w)}(y)}{\sum_{i\in\Z} u^{(i)}
      f^{(i)}(y)}\,dy
    \geq \frac{1}{\max_{i\in\Z,\,y\in\Y} f^{(i)}(y)} \int_{\Y} f^{(z)}(y)
    f^{(w)}(y)\,dy,    
  \end{equation}
  since the denominator can be upper bounded. Let $M$ denote the
  $m\times m$ matrix with its $(z,w)$ component equal to the
  right-most term in \eqref{eq:firstmatrixbound} above. Since it can
  be written as a Gram matrix for an integral product space, $M$ is
  positive semi-definite, and it is clear that $M \preceq -\nabla^2
  \varphi(u\,;\,f)$ for all $u\in\Delta^m$. Correspondingly, let
  $r(f)$ be the smallest eigenvalue of $M$.

  For an upper bound, we employ a different approach. Let $q_R(u) = C
  - (R/2)\, \|u-u_*\|^2$ denote the implied quadratic lower bound to
  $\varphi$. It is clear that $q_R(u) \geq 0$ if and only if
  $\|u-u^*\| \leq \sqrt{2C/R}$. Since $\varphi$ is a non-negative
  function, we only need to find $R$ so that $q_R$ is a lower bound
  when $q_R(u) > 0$. Consider
  \begin{align*}
    \inf_R &\quad R \\
    \text{s.t.} &\quad q_R(u) \leq \varphi(u\,;\,f) \quad \forall u:\,\|u -
    u_*\| < \sqrt{2C/R}.
  \end{align*}
  The problem is feasible since $\nabla^2 \varphi$ is continuous about
  $u_*$, and hence, there exists an $R$ sufficiently large such that
  $q_R$ is a lower bound of $\varphi$ in a small neighborhood around
  $u_*$. The problem is obviously bounded since the optimal value must
  be greater than $r(f)$. Now let $R(f)$ denote the optimal value to
  the problem above. Taylor expanding about $u_*$ yields
  \begin{equation*}
    r(f)\cdot\|u-u_*\|^2 \leq C(f) - \varphi(u\,;\,f) +
    \nabla_u\,\varphi(u_*\,;\,f)^T (u-u_*)
    \leq R(f)\cdot\|u-u_*\|^2.
  \end{equation*}
  But since $u_*>0$, optimality requires $\nabla_u\,\varphi(u_*\,;\,f)
  = \beta e$ for some $\beta\in\mathbb{R}$. Since $u$ and $u_*$ are
  both probability distributions,
  \begin{equation*}
    \nabla_u\,\varphi(u_*\,;\,f)^T (u-u_*) = \beta e^T (u-u_*) = 0,
  \end{equation*}
  and hence the lower and upper bounds hold.
  
  The proof for the discrete noise channel case is similar, with the
  exception being that we can easily find constants that satisfy the
  quadratic lower and upper bounds of the optimality gap $C(P) -
  \varphi(u_*\,;\,P)$. We observe the elements of $u^T P$ are lower
  bounded by $\kappa_1 = \min_{zy} P_{zy}$ and upper bounded by $\kappa_2 =
  \max_{zy} P_{zy}$. Therefore, for all $u \in \Delta^m$,
  \begin{equation*}
    \xi \kappa_2^{-1}\,P P^T \preceq \nabla_u^2\, \varphi(u\,;P) \preceq \xi \kappa_1^{-1}\, P P^T.
  \end{equation*}
  Lemma~\ref{lem:channelgrad} implies that $\nabla_u\varphi(u_*\,|\,P)
  = \beta e$ since $u_* > 0$. Thus, Taylor expansion about $u_*$
  yields
  \begin{equation*}
    \frac{\xi}{2\kappa_1}\left\|(u-u_*)^T P\right\|^2 \leq C(f) - \varphi(u;P)
    + \nabla_u\,\varphi(u_*;P)^T (u-u_*)
    \leq \frac{\xi}{2\kappa_2} \left\|(u-u_*)^T P\right\|^2.
  \end{equation*}
  Lastly, since both $u_*$ and $u$ are distributions, their components
  sum to one, implying $\nabla_u\,\varphi(u\,;\,P)^T (u - u_*) =
  0$. The result directly follows.
\end{proof}
This gives us explicit bounds on the entropy reduction in terms of the
$L_2$ distance of a question's predictive distribution from the
optimal predictive distribution. In theory, this allows us to
enumerate through all questions in $\X^m$ and select that whose
predictive distribution is closest to optimal, although this is
difficult when the size of $\X$ is large.

\subsubsection{Symmetric Noise Channel}
\label{subsubsec:symmetricnoise}

A symmetric noise channel is a special case of a discrete noise
channel, where the transmission matrix entries only depend on whether
or not $y=z$. There are many instances where in a moment of
indecision, the user can select an alternative uniformly at random,
especially when she does not have a strong opinion on any of the
presented alternatives. A symmetric noise channel useful for modeling
situations when $m$ is relatively small; if $m$ is large, with the
offered alternatives being presented as a list, the positioning in the
list might have an effect on the user's response. However, if the
number of alternatives in the comparative question is small, the
ordering should not matter, and a symmetric noise channel would be a
reasonable modeling choice. 

One way to parameterize a symmetric noise channel is by representing
the transmission matrices as $P_\alpha = \alpha I + (1-\alpha)(1/m)\,e
e^T$, where $e$ is a vector of all ones, and $\alpha \in [0,1]$. There
are other scenarios including symmetric noise channels that allow
$P^{(zy)} > P^{(zz)}$ for $y\neq z$, but these situations would be
particularly pessimistic from the perspective of learning, so we opt
to exclude these noise channels from our definition. Since
$\varphi(\cdot\,;\,P_\alpha)$ is concave and now symmetric in its
first argument, choosing $u^{(z)}_* = 1/m$ for every $z \in \Z$ is an
optimal solution. Thus, we want to choose the question $X_k$ so that
the user is equally likely to choose any of the offered alternatives.

In the case of symmetric noise, we can easily calculate the channel
capacity using \eqref{eq:channeleq}, yielding
\begin{equation}
  \label{eq:symchanneleq}
  C(P_\alpha) = \log_2 m - h\left(\alpha e^{(1)} + (1-\alpha)\,(1/m)\,e \right),
\end{equation}
where $e^{(1)}$ is an $m$-vector with its first component equal to
one, and all others equal to zero. The concavity of $h$ gives a crude
upper bound for the channel capacity, namely $C(P_\alpha) \leq \alpha
\log_2 m$. Comparatively, under no noise, one can reduce the entropy
of the posterior of the linear classifier by $\log_2 m$ bits at every
time epoch. There is an intuitive explanation for this result. With
noise level $\alpha$, we only observe the model-consistent response
with probability $\alpha$ at each step. Even under the best case
scenario of knowing which responses were model-consistent and which
were a random draw, the expected number of bits of reduced posterior
entropy at each step would only be $\alpha \log_2 m$. In fact, the
expected entropy reduction in reality is lower than this because we do
not know which responses are informative of linear classifier $\bm
\theta$.

Because the symmetric noise channel is a special case of a discrete
noise channel, we can leverage the results from
Theorem~\ref{thm:taylorbounds} to derive symmetric noise channel
sensitivity bounds.
\begin{corollary}
  Suppose we have a symmetric noise channel parameterized by
  $P_\alpha$, where $P_\alpha = \alpha I + (1-\alpha)(1/m)\, ee^T$,
  implying that $u^{(z)}_* = 1/m$ for all $z$. Then
  \begin{equation*}
    \frac{\xi\alpha^2}{2\left(\alpha + (1-\alpha)(1/m)\right)}
    \left\|u-u_*\right\|^2 \;\leq\; C(P_\alpha) -
    \varphi(u\,;P_\alpha)
    \;\leq\; \frac{\xi\alpha^2}{2(1-\alpha)(1/m)}\left\|u-u_*\right\|^2
  \end{equation*}
  for all $u \in \Delta^m$. 
\end{corollary}
\begin{proof}
  We start with the bounds from Theorem~\ref{thm:taylorbounds} and
  further refine. The off-diagonal entries of $P_\alpha$, by our
  parameterization of symmetric noise channel, are its smallest
  elements, and therefore, $\kappa_1 = (1-\alpha)(1/m)$. Similarly,
  the diagonal entries of $P_\alpha$ are the largest elements, and so
  $\kappa_2 = \alpha + (1-\alpha)(1/m)$. Lastly, one can easily verify
  $(u-u_*)^T P = \alpha \, (u-u_*)^T$.
\end{proof}
We return to the symmetric noise channel case in
Section~\ref{subsec:continuum}, where we show that in the theoretical
case of allowing fabrication of alternatives, a subset of alternatives
can always be constructed to achieve a uniform predictive
distribution regardless of the prior, and hence the optimal rate of
entropy reduction can always be achieved.

\subsection{Selection of Alternatives from the Continuum}
\label{subsec:continuum}

Now that we have results relating the predictive distribution
$u_k(X_k)$ to the entropy reduction in the linear classifier $\bm
\theta$, we now explore how we can appropriately choose alternatives
$X_k$ at every time epoch that yield a desirable predictive
distribution.

We first focus on the easier case where we can construct alternatives
to ask any comparative questions we desire. For a set of $m$
alternatives $(x^{(1)},\dots,x^{(m)}) = X \in \X^m$ and a prior
probability measure $\mu$, the characteristic polytopes
$A^{(1)}(X),\dots,A^{(m)}(X)$ determine the predictive
probabilities. Each set $A^{(z)}(X)$ composed of constraints $\theta^T
\left(x^{(z)}-x^{(i)}\right) \geq 0$ for $i\neq z$ (ignoring
strictness vs. non-strictness of inequalities). Thus, for the set of
alternatives $\X$ to have full expressiveness with respect to our
model, one must be able to choose alternatives so that $x^{(i)} -
x^{(j)}$ can take any direction in $\mathbb{R}^d$. A reasonable and
sufficient condition for the interior of $\X$ to be non-empty. When
this is the case, we can always choose alternatives such that the
relative direction between any two can take any value. This is what we
refer to as the continuum regime.

\begin{figure}[t]
    \centering
    \includegraphics[width=0.3\textwidth]{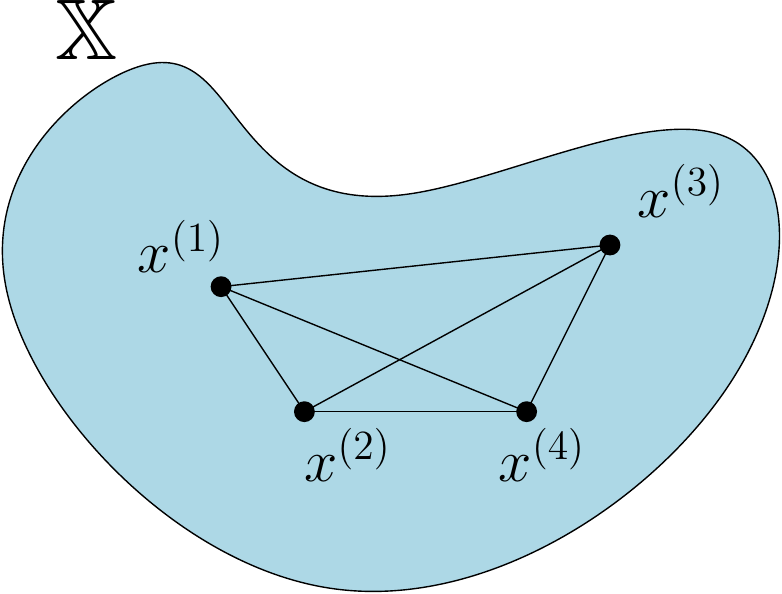}
    \caption{The continuum regime for alternative set $\X$. The
      document vectors $x^{(i)}$ can be chosen so that any direction
      $x^{(i)}-x^{(j)}$ can be achieved for all possible
      combinations.}
    \label{fig:interior}
\end{figure}

In most practical situations, the set of alternatives is finite, and
such construction is not possible. However, this assumption is more
mathematically tractable and allows us to give conditions for when we
can ask questions that yield a desirable predictive distribution, and
consequently maximize entropy reduction. We return to the more
realistic assumption of a finite alternative set later in
Section~\ref{sec:computation}.

Consider using pairwise comparisons, i.e., when $m=2$. Is it true that
regardless of the noise channel and the prior distribution of
$\bm\theta$ that we can select a question $X_k$ that achieves the
optimal predictive distribution $u_k(X_k)$? A simple example proves
otherwise. Suppose \textit{a priori}, the linear classifier
$\bm\theta$ is normally distributed with zero mean and an identity
covariance matrix. Because the distribution is symmetric about the
origin, regardless of the hyperplane we select, exactly $1/2$ of the
probabilistic mass lies on either side of the hyperplane. This is the
desirable outcome when the noise channel is symmetric, but suppose
this were not the case. For example, if the noise channel required
$2/3$ of the probabilistic mass on one side of the hyperplane, there
is no way to achieve this. 

This issue is related to a certain metric called \textit{halfspace
  depth}, first defined by \citet{tukey1975mathematics} and later
refined by \citet*{donoho1992breakdown}.  The halfspace depth at a
point $\eta \in \mathbb{R}^d$ refers to the minimum probabilistic mass
able to be partitioned to one side of a hyperplane centered at
$\eta$. In this paper, we only consider the case where the cutting
plane is centered at the origin, and need only to consider the case
where $\eta = 0$. Hence, let
\begin{equation}
  \label{eq:halfspacedepth}
  \delta(\mu_k)
  = \inf_{v\neq 0} \mu_k\left(\left\{\theta:\,\theta^T v
  \geq 0 \right\} \right).
\end{equation}
In our previous example, the halfspace depth of the origin was equal
to $1/2$, and therefore, there were no hyperplanes that could
partition less than $1/2$ of the probabilistic mass on a side of a
hyperplane.

The question now is whether we can choose a hyperplane such that
$u_k^{(z)}(X_k) = u^{(z)}_*$ for any $u^{(z)}_* \in
\left[\delta(\mu_k),1-\delta(\mu_k)\right]$. We first prove an
intuitive result regarding the continuity of probabilistic mass of a
halfspace with respect to the cutting plane. One can imagine rotating
a hyperplane about the origin, and since the probability measure has a
density with respect to Lebesgue measure, there will not be any sudden
jumps in probabilistic mass on either side of the hyperplane.

\begin{lemma}
  \label{lem:halfspacecontinuous}
  If probability measure $\mu$ is absolutely continuous with respect
  to Lebesgue measure, then the mapping $v \mapsto
  \mu\left(\{\theta\in\Theta:\,\theta^T v \geq 0\}\right)$ is
  continuous.
\end{lemma}
\begin{proof}
  Suppose we have a sequence $(v_j:\,j\geq 0)$ in $\mathbb{R}^d
  \setminus \{0\}$ such that $v_j \to v$. The functions
  $\mathbb{I}(\{\theta:\,\theta^T v_j \geq 0\})$ converge to
  $\mathbb{I}(\theta:\,\theta^T v \geq 0\})$ almost surely. Taking
  expectations and using Dominated Convergence Theorem gives the result.
\end{proof}

Lemma~\ref{lem:halfspacecontinuous}
enables us to find conditions under which we can ask a
question $X_k$ that yields a desirable predictive distribution
$u_k(X_k)$.  In particular, Corollary~\ref{cor:halfspaceselection}
uses a variant of the intermediate value theorem.
\begin{corollary}
  \label{cor:halfspaceselection}
  Suppose $u_* > 0$ and $\Int(\X) \neq \varnothing$. Then there exists
  $X_k = (x_1,x_2)\in\X^2$ such that $u_k(X_k) = u_*$ if and only if
  $\max u_* \leq 1-\delta(\mu_k)$.
\end{corollary}
\begin{proof}
  Take any $v\in C = \{w\in\mathbb{R}^d:\,\|w\| = 1\}$, where
  $\mu_k\left(\{\theta:\,\theta^T v \geq 0 \}\right) =
  \delta(\mu_k)$. Now let $v' = -v$, and since $\mu_k$ is absolutely
  continuous with respect to Lebesgue measure, $\mu_k(\theta^T v' \geq
  0) = \mu_k(\theta^T v' > 0) = 1-\delta(\mu_k)$. Also, $C$ is
  connected, and $w \mapsto \mu_k(\{\theta:\,\theta^T w \geq 0\})$ is
  a continuous mapping: it follows that the image of any path from $v$
  to $v'$ must also be connected. But the image is a subset of the
  real line, and therefore must be an interval. Lastly, $C$ is a
  compact set, implying that the endpoints of this interval are
  attainable, and so the image of any such path is equal to
  $[\delta(\mu_k),1-\delta(\mu_k)]$.

  To recover the two alternatives, first select a vector
  $w\in\mathbb{R}^d$ such that $\mu_k(\{\theta:\,\theta^T w \geq 0\})
  = u^{(1)}$. Choose $x^{(1)} \in \Int(\X)$, and subsequently
  choose $x^{(2)} = x^{(1)} - cw$, where $c>0$ is a positive scalar
  that ensures $x^{(2)} \in \X$. Finally, let $X_k =
  \left(x^{(1)},x^{(2)}\right)$.

  \begin{figure}[t]
    \centering
    \includegraphics[width=0.3\textwidth]{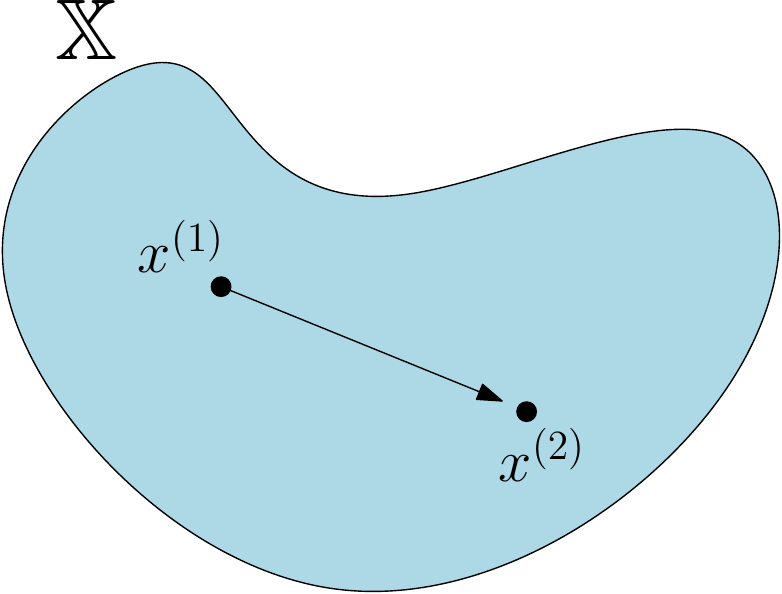}
    \caption{Selection of alternatives $x^{(1)}$ and $x^{(2)}$. Since
      $x^{(1)}$ lies in the interior of $\X$, it is always possible to
      choose $x^{(2)}$ so that $x^{(1)} - x^{(2)}$ has the same
      direction as $v$.}
    \label{fig:intpairdefine}
  \end{figure}

  To prove the converse statement, suppose $\max\{u_*\} >
  1-\delta(\mu_k)$. Then by definition of halfspace depth,
  $\min\{u_*\} \notin \left\{\mu_k(\{\theta:\,\theta^T v \geq 0
  \}):\,v \neq 0\right\}$. Thus, there does not exist a hyperplane
  that can separate $\mathbb{R}^d$ into two halfspaces with
  probabilistic mass $u_*$.
\end{proof}

Can we draw a similar conclusion if we offer more more than two
alternatives at each time epoch?  The mass partition problem becomes
increasingly complex when greater than two alternatives are
included. Since the sets $A^{(z)}(X)$ correspond to convex polyhedral
cones, the problem becomes that of finding a partition of $m$ convex
polyhedral cones, or a polyhedral $m$-fan as it is known in the
computational geometry literature, that attains the prescribed
probabilistic mass $u_*$. There are a number of results pertaining to
convex equipartitions and extensions of the Borsuk-Ulam Theorem, most
notably the Ham Sandwich Theorem. Despite this, to the best of our
knowledge, there is no result for general mass partitions of convex
polyhedral $m$-fans in the computational geometry literature. For this
reason, we prove such a result here: that one can construct a
polyhedral $m$-fan with the corresponding predictive distribution
$u_*$ if the measure $\mu$ is such that $\max\{u_*\} < 1-\delta(\mu)$.

Unlike the previous case that focused on pairwise comparisons, the
inequality is strict. One of the reasons this is the case is because
of the specific structure of the polyhedral cones in our
problem. Since $A^{(z)}(X)$ corresponds to the linear classifier in
which the dot product with alternative $z$ is maximal, these
polyhedral cones cannot be halfspaces unless the predictive
probability for some alternative equals zero, which we do not
allow. Thus, we enforce the additional constraint that each $A^{(z)}$
is a \textit{salient} convex polyhedral cone, meaning that it does not
contain a linear subspace.

To prove the result, we first show the result in the case of two
dimensions: constructing the polyhedral $m$-fan, then deriving the
feature vectors for the corresponding alternatives. This result is
then generalized to the case of any dimension by using a projection
argument. 

\begin{lemma}
  \label{lem:2dpolyfan}
  Suppose $d=2$ and $m>2$. If $\max\{ u_*\} < 1-\delta(\mu)$, then
  there exists a two-dimensional polyhedral $m$-fan characterized by
  polyhedral cones $(A^{(z)}:\,z\in\Z)$ such that $\mu(A^{(z)}) =
  u^{(z)}_*$ for all $z\in\Z$.
\end{lemma}
\begin{proof}
  Without loss of generality we can assume $\|\theta\| = 1$, and in
  the case of two dimensions that is equivalent to $\theta$ being
  parameterized by the interval $[0,2\pi)$ on the unit circle. For an
    interval $\mathcal{I}$ measuring angles in radians, let
  \begin{equation*}
    \text{Cone}(\mathcal{I}) = \left\{\left(\begin{array}{c}
      r\cos\eta \\ r\sin\eta \\ \end{array} \right):\,\eta\in
    \mathcal I,\,r>0\right\}.
  \end{equation*}
  Accordingly, let $\mu^C$ be a measure defined on the unit circle
  such that $\mu^C(\mathcal I) = \mu\left(\Cone(\mathcal I)\right)$ for every
  Lebesgue-measurable interval on $[0,2\pi)$. This implies that
    $\delta(\mu^C) = \delta(\mu)$. For radian angles $\eta^{(1)} <
    \eta^{(2)} < \dots < \eta^{(m+1)} = \eta^{(1)} + 2\pi$, we define
  \begin{equation*}
    B^{(z)} = \begin{cases}
      [\eta^{(1)},\eta^{(2)}] & z = 1 \\
      (\eta^{(z)},\eta^{(z+1)}] & z = 2,\dots,m-1\\
      (\eta^{(m)},\eta^{(m+1)}) & z = m\\ \end{cases}
  \end{equation*}
  for all $z\in\Z$. The asymmetry with respect to sets being $B^{(z)}$
  closed or open is due to the definition of $A^{(z)}$ in
  \eqref{eq:characteristicpolytopes}. For each $B^{(z)}$ to correspond
  to a convex set strictly contained in a halfspace, we require
  $\eta^{(z+1)} - \eta^{(z)} < \pi$. Our objective is to appropriately
  select the angles $(\eta^{(z)}:\,z\in\Z)$ so that $\mu^C(B^{(z)}) =
  u^{(z)}$. It suffices to consider only two cases. \\
    
  \textbf{Case
    1: $\max\{u_*\} < \delta(\mu)$}\\
 
  This is the simpler case, since all the probabilities from the
  predictive distribution are strictly smaller than \textit{any}
  halfspace measure. Arbitrarily choose $\eta^{(1)}$. Now we want to
  choose $\eta^{(2)} \in \left(\eta^{(1)},\eta^{(1)}+\pi\right)$ so
  that the interval contains prescribed measure $u_*^{(1)}$. The
  function $\eta^{(2)} \mapsto \mu^C\left( [\eta^{(1)},\eta^{(2)}]
  \right) $ is monotonically increasing, continuous, and takes values
  on $(0,\delta(\mu)]$. Since $u_*^{(z)} \leq \max u_* \leq
  \delta(\mu_k)$, the Intermediate Value Theorem allows us to choose
  $\eta^{(2)}$ so the interval has measure $u_*^{(1)}$. Continue in
  this way until all such angles $\eta^{(z)}$ are attained. \\

  \textbf{Case 2: $\max\{u_*\} \in
    \left[\delta(\mu),1-\delta(\mu)\right)$} \\

  Here, it is necessary to define the set $B^{(z)}$ corresponding to
  the largest predictive probability $u_*$ first, and that with the
  smallest second. Without loss of generality, suppose $u_*^{(1)} =
  \max u_*$ and $u_*^{(2)} = \min u_*$. Then choose $\eta^{(1)}$
  such that
  \begin{equation}
    \label{eq:pizzacondition}
    \mu^C\left([\eta^{(1)},\eta^{(1)}+\pi]\right)
    \in \left(\max u_*\,,\, \max u_* + \min u_*\right).
  \end{equation}
  This is possible because
  \begin{equation*}
    \left( \delta(\mu), 1-\delta(\mu) \right) \cap
    \left( \max u_*\,,\,\max u_* + \min u_* \right) \neq \varnothing,
  \end{equation*}
  due to the assumptions that $\max u_* \in
  \left[\delta(\mu),1-\delta(\mu)\right)$ and $\min u_* > 0$. Now
  define $\eta^{(2)}$ such that $\mu^C
  \left[\eta^{(1)},\eta^{(2)}\right] = \max u_*$.

  Now we define the interval corresponding to $\min u_*$ directly
  adjacent. Suppose $\mu^C(\eta^{(2)},\eta^{(2)}+\pi] > \min
  u_*$. Then by the Intermediate Value Theorem, there exists some
  $\eta^{(3)}$ such that $\mu^C(\eta^{(2)},\eta^{(3)}] = \min
  u_*$. Otherwise, suppose that $\mu^C(\eta^{(m+1)}-\pi,\eta^{(m+1)})
  > \min u_*$. Again, by the Intermediate Value Theorem, we can find
  $\eta^{(m)}$ less than $\pi$ radians from $\eta^{(m+1)}$ such that
  $\mu^C(\eta^{(m)},\eta^{(m+1)}) = \min u_*$.

  We claim that these are the only two possibilities. By way of
  contradiction, suppose that neither of these scenarios are true; in
  other words,
  \begin{align*}
    \mu^C[\eta^{(2)},\eta^{(2)}+\pi] &\leq \min u_* \\
    \mu^C[\eta^{(m+1)}-\pi,\eta^{(m+1)}] &\leq \min u_*.
  \end{align*}
  We can decompose these intervals into non-overlapping parts. Define
  \begin{align*}
    a &= \mu^C(\eta^{(2)}+\pi,\eta^{(1)}+2\pi] \\
    b &= \mu^C(\eta^{(2)},\eta^{(1)}+\pi] \\
    c &= \mu^C(\eta^{(1)}+\pi,\eta^{(2)}+\pi).    
  \end{align*}
  \begin{figure}[t]
    \centering
    \label{fig:minpizza}
    \includegraphics[width=0.5\textwidth]{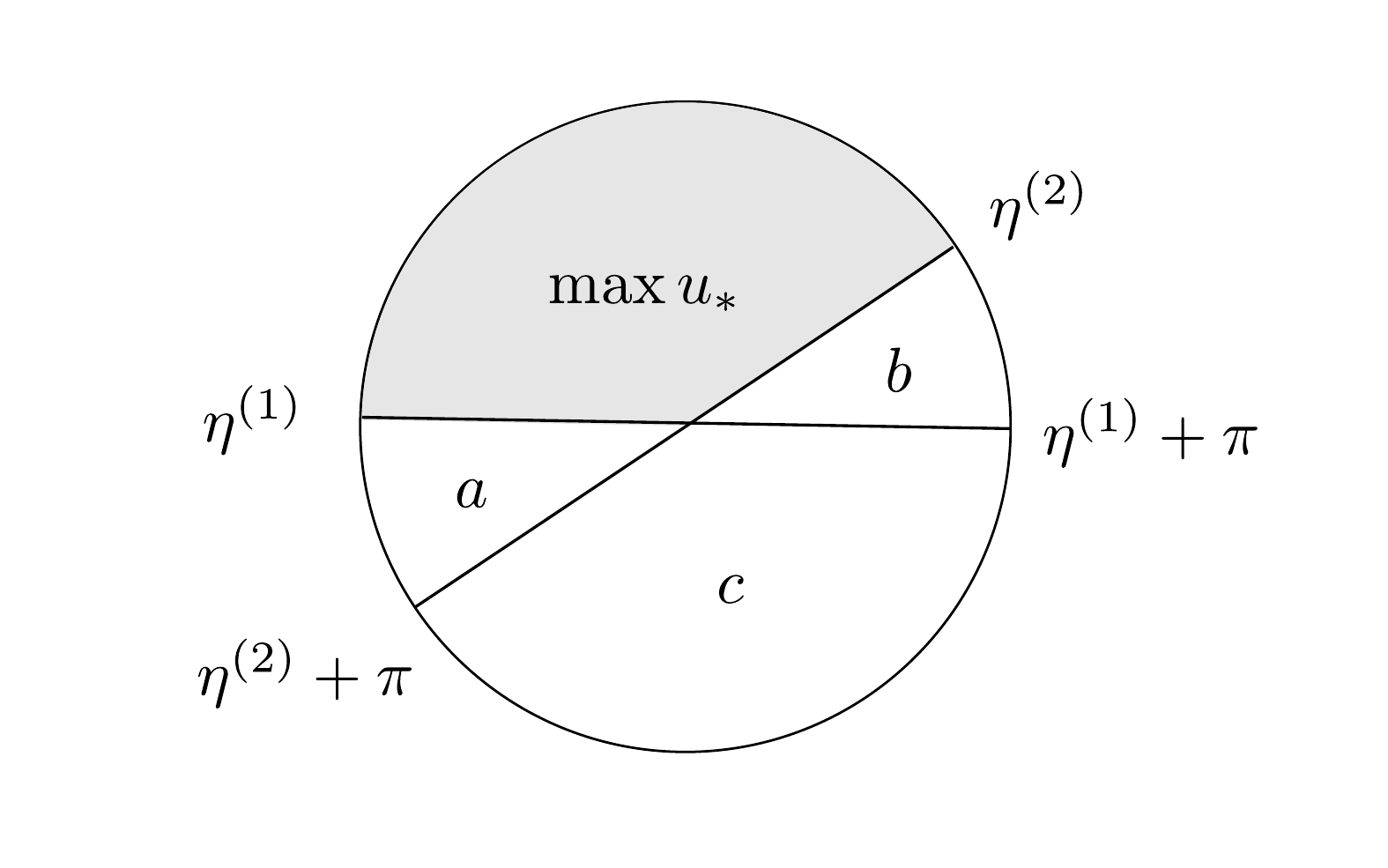}
    \caption{Diagram showing the distribution of probabilistic mass
      partitioned in unit circle.}
  \end{figure}
  Suppose that $\max\{a,b\} + c \leq \min u_*$. The measure of the
  union of the three intervals $1-\max u_* = a + b + c$, which implies
  $1-\max u_* \leq \min u_* + \min\{a,b\}$.  Finally, since the
  smallest component of $u_*$ must be smaller in magnitude than the sum
  of the other non-maximal components,
  \begin{equation*}
    \max\{a,b\} + c \leq \min u_* \leq 1-\max u_* -\min u_* \leq \min\{a,b\},
  \end{equation*}
  implying among other things that $b = \min u_*$ in this
  scenario. However, this is a contradiction, since we originally
  chose $\eta^{(1)}$ such that $b+c < \max u_* + \min u_*$ due to
  \eqref{eq:pizzacondition}. Therefore, this scenario is not possible,
  and we can always find an interval with probabilistic mass strictly
  greater than $\min u_*$ directly adjacent to an interval with
  maximal probabilistic mass.
      
  In all cases, we have defined the first two intervals, and the
  remaining unallocated region of the unit circle is strictly
  contained in an interval of width less than $\pi$ radians. Thus, one
  can easily define a partition as in Case 1, and every subsequent
  interval would necessarily have to have length strictly less than
  $\pi$ radians. To recover the convex cones, let $A^{(z)} =
  \Cone\left(B^{(z)}\right)$ for every $z\in\Z$, and it is clear that
  $A^{(z)}$ contains the desired probabilistic mass.
\end{proof}

Lemma~\ref{lem:2dpolyfan} gives a way to construct polyhedral fans
with the desired probabilistic mass. We are interested in finding a
set of alternatives that represents this polyhedral fan, and this is
exactly what Theorem~\ref{thm:2dconstruction} does in the
two-dimensional case. The critical condition required is for the set
of alternatives $\X$ to have non-empty interior.

\begin{theorem}
  \label{thm:2dconstruction}
  Suppose $d = 2$ and $m > 2$. Then given a measure $\mu$ that is
  absolutely continuous with respect to Lebesgue measure and an
  optimal predictive distribution $u_*$, if $\Int(\X) \neq
  \varnothing$ and $\max u_* < 1-\delta(\mu)$, then there exists $X
  \in \X^m$ such that $u(X) = u_*$.
\end{theorem}
\begin{proof}
  First, use Lemma~\ref{lem:2dpolyfan} to construct a polyhedral fan
  with the correct probabilistic weights.  Using the angles
  $\eta^{(1)},\dots,\eta^{(m)}$ constructed in the Lemma, we
  can define separating hyperplanes $v^{(1)},\dots,v^{(m)}$ by setting
  $v^{(z)} = \left( -\sin \eta^{(z)},\,\cos
  \eta^{(z)}\right)$. Then we have
  \begin{equation*}
    \bar A^{(z)} = \left\{\theta:\,\begin{array}{c}
    \theta^T v^{(z)} > 0 \\
    \theta^T v^{(z+1)} \leq 0 \\ \end{array} \right\}.
  \end{equation*}
  The goal now is to define the alternatives. First, choose $x^{(1)}
  \in \Int(\X)$. Now define $x^{(z+1)} = x^{(z)} + c^{(z+1)}
  v^{(z+1)}$, where $c^{(z+1)} > 0$ is a positive scaling that ensures
  $x^{(z+1)}\in \Int(\X)$ if $x^{(z)} \in \Int(\X)$. Now we can
  equivalently write

  \begin{figure}[t]
    \centering
    \label{fig:intmultidefine}
    \includegraphics[width=0.3\textwidth]{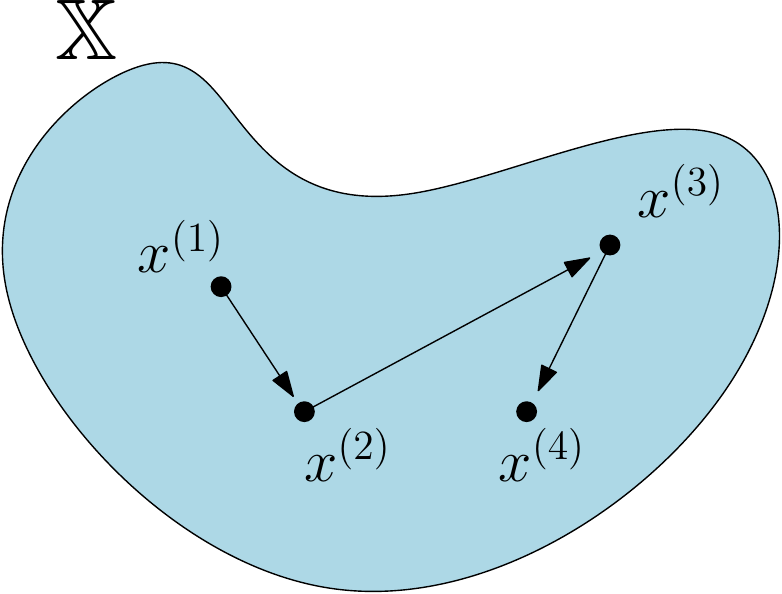}
    \caption{Iterative selection of alternatives. Since each $x^{(z)}$
      is in the interior of $\X$, it is always possible to select
      $x^{(z+1)}$ to maintain a specific direction for
      $x^{(z+1)}-x^{(z)}$.}
  \end{figure}
  
  \begin{equation*}
    \bar A^{(z)} = \left\{\theta:\,\begin{array}{c}
    \theta^T\left(x^{(z)} - x^{(z-1)}\right) > 0 \\
    \theta^T\left(x^{(z)} - x^{(z+1)}\right) \geq 0. \\
    \end{array} \right\}.
  \end{equation*}
  Let $X = (x^{(1)},\dots,x^{(m)})$. It remains to show that
  $A^{(z)}(X) = \bar A^{(z)}$. Because $A^{(z)}(X)$ has the same linear
  inequalities as $\bar A^{(z)}$, it is clear that $A^{(z)}(X) \subseteq
  \bar A^{(z)}$ for all $z$. Now suppose there exists some $\theta \in
  \bar A^{(z)}$. Since $(\bar A^{(z)}:\,z\in\Z)$ is a partition of
  $\mathbb{R}^2$, it is clear that $\theta \notin A^{(z')}$ for
  $z'\neq z$, and thus, $\theta \notin A^{(z')}(X)$. Since
  $\left(A^{(z)}(X):\,z\in\Z\right)$ is also a partition of
  $\mathbb{R}^2$, it must be that $\theta\in A^{(z)}(X)$. This
  directly implies $\bar A^{(z)} = A^{(z)}(X)$, and so $u^{(z)}(X) =
  \mu(A^{(z)}) = u_*^{(z)}$.
\end{proof}
Theorem~\ref{thm:2dconstruction} shows that in the case of two
dimensions, a set of $m$ alternatives can be generated to ensure that
the entropy of the posterior distribution of $\bm\theta$ maximally
decreases. This result can be generalized to arbitrary dimension by
selecting a two dimensional subspace and leveraging the previous
result.

\begin{theorem}
  \label{thm:2dslices}
  Suppose $\Int(\X)\neq \varnothing$ and $u_* > 0$. If $\max u_*
  < 1-\delta(\mu_k)$, then there exists $X_k =
  (x^{(1)},x^{(2)},\dots,x^{(m)}) \in \X^m$ such that $u_k(X_k) =
  u_*$. Further, if $\max u_* > 1-\delta(\mu_k)$, then finding such a
  question is not possible.
\end{theorem}
\begin{proof}
  We begin by proving the last claim of the theorem. Since any
  $A^{(z)}(X_k)$ can be contained by a halfspace centered at the
  origin, and since all such halfspaces have probabilistic mass less
  than or equal to $1-\delta(\mu_k)$, then we must have
  $\mu_k(A^{(z)}(X)) \leq 1-\delta(\mu_k)$ for every $z\in\Z$ and for
  every $X_k\in\X^m$.
  
  Now we show the main result of the theorem. There exists some $\bar
  \beta \in\mathbb{R}^m\setminus \{0\}$ such that $\mu_k\left( \{\theta:\,
  \theta^T \bar \beta \geq 0\}\right) = \delta(\mu_k)$, since $\mu_k$ has
  density $p_k$ and is continuous. Let $\mathcal{H} = \{\theta:\,\theta^T
  \bar \beta = 0\}$ denote the hyperplane. Now choose a two-dimensional
  subspace $L$ such that $L \perp \mathcal{H}$. For $\nu \in L$,
  define density $p_k^L$ as
  \begin{equation*}
    p_k^L(\nu) = \int_{\omega \in L^\perp}
    p_k(\nu + \omega)\,\lambda_{d-2}(d\omega),
  \end{equation*}
  where $\lambda_{d-2}$ is $(d-2)$-Lebesgue measure, and let $\mu_k^L$
  denote measure induced by density $p_k^L$. For $\beta \in L$, we
  have
  \begin{align*}
    \mu_k^L\left(\{\nu\in L:\, \nu^T \beta \geq 0\}\right)
    &= \int_{\nu\in L:\, \nu^T \beta\geq 0}
      p_k^L(\nu)\,\lambda_2(d\nu) \\
    &= \int_{\nu\in L:\, \nu^T \beta\geq 0}
      \int_{\omega\in L^\perp}
      p_k(\nu+\omega)\,\lambda_{m-2}(d\omega) \,\lambda_2(d\nu) \\
    &= \int_{(\nu,\omega) \in (L\times L^\perp):\,
        \left(\nu + \omega\right)^T \beta \geq 0}
      p_k(\nu+\omega)\,\lambda_d(d\nu\times d\omega) \\
    &= \int_{\theta:\,\theta^T \beta\geq 0} p_k(\theta)\,
      \lambda(d\theta) = \mu_k(\{\theta:\,\theta^T \beta \geq 0\}).        
  \end{align*}
  Thus, $\mu_k^L$ is consistent with $\mu_k$. In particular,
  $\mu_k^L(\{\theta:\, \theta^T \bar \beta \geq 0\}) = \delta(\mu_k)$,
  and thus $\delta(\mu_k^L) \leq \delta(\mu_k)$, meaning we can use
  the previous Theorem~\ref{thm:2dconstruction} to find an appropriate
  comparative question.

  In particular, let $\gamma_1$ and $\gamma_2$ denote two orthogonal
  $d$-vectors that span $L$, and let $\Gamma \in \mathbb{R}^{d\times
    2}$ contain $\gamma_1$ and $\gamma_2$ as its columns. To use the
  Theorem~\ref{thm:2dconstruction}, we pass $\mu_k^L \circ \Gamma$,
  and to convert the resulting question $X_k^L =
  (x^{(1)},\dots,x^{(m)})$ back into $d$-dimensional space, take $X_k
  = (\Gamma x^{(1)},\dots,\Gamma x^{(m)})$.
\end{proof}
Theorem~\ref{thm:2dslices} provides one possible construction for a
question $X_k$ that gives a desirable predictive distribution,
although there may be others. However, it is clear that if the
halfspace depth $\delta(\mu_k)$ is too large, it will not always be
possible to find a question that can yield the optimal predictive
distribution, even if we can construct questions in the continuum. But
while it may not be possible to \textit{maximally} reduce the entropy
of the posterior distribution, we may choose a question $X_k$ that can
still reduce entropy by a constant amount at each time epoch.

We conclude the section by showing that entropy in the linear
classifier $\bm \theta$ can be reduced linearly, even if not optimally.
\begin{theorem}
  \label{thm:continuumlower}
  Suppose $\X = \mathbb{R}^d$, and let $\bm\sigma_k =
  \left(\max\{u_*\} - \left(1 - \delta(\bm\mu_k)
  \right)+\epsilon\right)^+$. Then the following upper bound holds
  when $m=2$ for $\epsilon = 0 $ and when $m>2$ for arbitrarily small
  $\epsilon > 0$.
  \begin{align*}
    K\cdot C(f) - \sup_\pi \mathbb{E}^\pi \left[
      I(\bm\theta;\Yhist{k}) \right] &\leq
    r(f) \left(1 + \frac{1}{m-1}\right) \sum_{k=1}^K
    \mathbb{E}^\pi[\bm\sigma_k^2] \\
    &\leq K\cdot r(f) \left(1 + \frac{1}{m-1}\right)
    \left((\max\{u_*\} - 1/2 + \epsilon)^+\right)^2.
  \end{align*}  
\end{theorem}
\begin{proof}
  We start with the case when $m > 2$. Fix any small $\epsilon >
  0$. Let $z' = \argmax\{u_*\}$. We write equality because in the
  cases where $\bm\sigma_k > 0$, the maximum component is unique;
  otherwise, when $\bm \sigma_k = 0$, the choice of $z'$ is
  irrelevant. We construct an ``approximate predictive distribution''
  $\bar u_k$ such that
  \begin{equation*}
    \bm{\bar u}_k^{(z)} = \begin{cases}
      u_*^{(z)} - \bm\sigma_k & z = z' \\
      u_*^{(z)} + \bm\sigma_k / (m-1) & z \neq z'\\
    \end{cases}
  \end{equation*}
  This new vector $\bm{\bar{u}}_k$ is the projection of $u_*$ onto the
  set $\{u\in\Delta^m:\,\max\{u\} \leq (1-\delta(\mu_k)) -\epsilon\}$. This
  ``approximate predictive distribution'' is chosen to minimize the
  $L_2$ distance from optimal $u_*$, and therefore maximize entropy
  reduction.

  One can show that $\max\{\bar u_k\} < 1-\delta(\mu_k)$, and $\|\bar
  u_k - u_*\|^2 \leq \sigma_k^2 \left( 1 + 1/(m-1) \right)$. Now we
  can construct $\bar X_k$ such that $u_k(\bar X_k) = \bar u_k$ at every
  step, which is possible by Theorem~\ref{thm:2dslices} since $\bar u
  > 0$ and $\max\{\bar u_k\} < 1- \delta(\mu_k)$. Now we use
  Theorem~\ref{thm:taylorbounds} to show
  \begin{align*}
    K\cdot C(f) - \sup_\pi I^\pi(\bm\theta;\Yhist{k})
    &\leq \sum_{k=1}^K \left(C(f) - \mathbb{E}_k^\pi\left[\varphi(u_k(\bar X_k)\,;\,f)\right]\right) \\
    &= \sum_{k=1}^K \left(C(f) - \mathbb{E}_k^\pi\left[\varphi(\bm{\bar u_k}\,;\,f)\right]\right) \\
    &\leq \sum_{k=1}^K r(f) \,\mathbb{E}_k^\pi\left[\|\bm{\bar u_k} - u_*\|^2 \right]\\
    &= r(f) \left(1 + \frac{1}{m-1}\right) \sum_{k=1}^K \mathbb{E}^\pi_k\left[\bm\sigma_k^2\right].
  \end{align*}
  And since $1-\delta(\mu_k) \geq 1/2$, it follows that $\sigma_k \leq
  \left(\max\{u_*\} - 1/2 + \epsilon\right)^+$, regardless of
  $\mu_k$. The proof is analogous for the $m=2$ case: the only change
  required is to set $\epsilon = 0$, because by
  Corollary~\ref{cor:halfspaceselection}, we can find a question if
  $\max\{u_*\} \leq 1- \delta(\pi_k)$, where the inequality need not be
  strict. 
\end{proof}

Putting Theorem~\ref{thm:continuumlower} together with
Corollary~\ref{cor:kstepentropy} shows that if the alternative set
$\X$ has non-empty interior, the expected differential entropy of
linear classifier $\bm\theta$ can be reduced at a linear rate, and
this is optimal up to a constant factor.

Finally, recall in Section~\ref{subsubsec:symmetricnoise} we defined
the case of a symmetric noise channel. There, $u_*^{(z)} = 1/m$ for
all $z\in\Z$. In the pairwise comparison case, $\max u_* = 1/2 \leq
1-\delta(\mu)$ for all measures $\mu$. In the multiple comparison
case, $\max u_* = 1/m < 1/2 \leq 1-\delta(\mu)$ for all measures
$\mu$.  Thus, regardless of $m$, in the continuum setting, a set of
alternatives can always be constructed to achieve a uniform predictive
distribution, and therefore optimally reduce posterior entropy.


\section{Misclassification Error}
\label{sec:misclass}

The entropy pursuit policy itself is intuitive, especially when the
noise channel is symmetric. However, differential entropy as a metric
for measuring knowledge of the user's preferences is not
intuitive. One way to measure the extent of our knowledge about a
user's preferences is testing ourselves using a randomly chosen
question and estimating the answer after observing a response from the
user. This probability we get the answer wrong called
misclassification error.

Specifically, we sequentially ask questions $X_k$ and observe signals
$\bm Y_k(X_k)$ at time epochs $k=1,\dots,K$. After the last question,
we are then posed with an evaluation question. The evaluation question
will be an $n$-way comparison between randomly chosen alternatives,
where $n$ can differ from $m$. Denote the evaluation question as $\bm
S_K \in \X^n$, where a particular evaluation question $S_K =
\left(s^{(1)},\dots,s^{(n)}\right)$. The evaluation question is chosen
at random according to some unknown distribution. Denote the
model-consistent answer as $\bm W_K(S_K) = \min\left\{\argmax_{w\in\W}
\bm\theta^T s_w\right\}$, where the minimum serves as a tie-breaking
rule. The goal is to use history $\bmYhist{K}$ and the question $\bm
S_K$ to predict $\bm W_K(S_K)$. Let $\hat W_K$ denote the candidate
answer that depends on the chosen evaluation question response
history. Then our goal for the adaptive problem is to find a policy
that minimizes
\begin{equation}
  \label{eq:misclassify}
  \mathcal{E}_K^\pi = \E^\pi \left[ \inf_{\hat W_K\in\W}
    \mathbb{P}\left(\bm W_K(\bm S_K) \neq
    \hat W_K \,\middle|\, \bmYhist{K}\,,\bm S_K\right)\right],
\end{equation}
and one can do this by adaptively selecting the best question $X_k$
that will allow us to learn enough about the user's preferences to
correctly answer evaluation question $\bm S_k$ with high
certainty. Let $\mathcal{E}_K^* = \inf_\pi \mathcal{E}_K^\pi$ be the
misclassification error under the optimal policy, assuming it is
attained.

We make several reasonable assumptions on the dependence of $\bm S_K$
and $\bm W_K(\bm S_K)$ with the model-consistent response $\bm
Z_k(X_k)$ and signal $\bm Y_k(X_k)$ from the learning question $X_k$.
\begin{assume}{Evaluation Question Assumptions}
  \label{assm:misclassifynoise}
  For evaluation question $\bm S_K = S_K$ and corresponding
  model-consistent answer $\bm W_K(S_K)$, we assume
  \begin{itemize}
    \item Evaluation question $\bm S_K = S_K$ is selected randomly
      from $\X^n$, independent from all else, and
    \item For all such questions $S_K$, signal $\bm Y_k(X_k)$ and
      model-consistent answer $\bm W_K(S_K)$ are conditionally independent
      given $\bm Z_k(X_k)$ for all $k=1,\dots,K$.
  \end{itemize}
\end{assume}

In practice, solving the fully adaptive problem is intractable, and
instead, one can use a knowledge gradient policy to approach this
problem. This is equivalent to solving a greedy version of the problem
where we are evaluated at every step. In other words, after observing
signal $\bm Y_k(X_k)$, we are posed with answering a randomly selected
evaluation question $\bm S_k$, with no concern about any future
evaluation. Every question in the sequence $\left(\bm
S_k:\,k=1,\dots,K\right)$ is selected i.i.d. and follows the
\ref{assm:misclassifynoise}. The knowledge gradient policy chooses
$X_k$ such that at every time epoch $k$ it solves
\begin{equation*}
  \mathcal{E}_k^{KG} = \inf_{X_k\in\X^m} \mathbb{E}^{KG}\left[ \inf_{\hat W_k\in\W}
    \mathbb{P}\left( \bm W_k(S_k) \neq \hat W_k \,\middle|\,
    \bmYhist{k},\,\bm S_k\right) \right].
\end{equation*}
Obviously, $\mathcal{E}_k^{KG} \geq \mathcal{E}_k^*$ for all $k$,
since knowledge gradient cannot perform strictly better than the fully
adaptive optimal policy. It would be beneficial to know how wide the
gap is, and this can be done by finding a lower bound on the optimal
misclassification error. Information theory provides a way to do this,
and in the next sections, we will show a lower bound in terms of the
entropy reduction of the underlying linear classifier, and that
posterior entropy reduction is necessary to achieve misclassification
error reduction.


\subsection{An Interactive Approach}

It would be helpful to have an analogue to
Theorem~\ref{thm:entropyidentities} so we can relate the posterior
Shannon entropy of the answer $\bm W(S)$ of evaluation question
$S$ to the answer $\bm Z_k(X_k)$ of initial question $X_k$. It turns
out that information content in observing signal $\bm Y_k(X_k)$ to
infer answer $\bm W_k$ is related to a concept in information theory
called \textit{interaction information}. In the context of this paper,
for a model consistent answer $\bm Z_k$, observed response $\bm Y_k$,
evaluation question $S$ and true answer $\bm W(S)$, Interaction
Information denotes the difference
\begin{align*}
    I_k(\bm W(S); \bm Y_k; \bm Z_k)
    &= I_k(\bm W(S);\bm Y_k\,|\,\bm Z_k) - I_k(\bm W(S);\bm Y_k) \\
    &= I_k(\bm Y_k;\bm Z_k\,|\,\bm W(S)) - I_k(\bm Y_k;\bm Z_k) \\
    &= I_k(\bm Z_k;\bm W(S)\,|\,\bm Y_k) - I_k(\bm Z_k;\bm W_k).
\end{align*}
Similarly, we define Conditional Interaction Information as
\begin{equation*}
    I_k(\bm W(\bm S);\bm Y_k;\bm Z_k\,|\,\bm S)
    = \E\left[ I_k(\bm W(\bm S);\bm Y_k;\bm Z_k\,|\,\bm S = S)\right].
\end{equation*}
Interaction information tells us the relationship between three random
variables in terms of the redundancy in information content. In
general, this quantity can be positive or negative. If the interaction
information between three random variables is negative, then one does
not learn as much from an observation when already knowing the outcome
of another. This is the more natural and relevant case in the context
of misclassification error.

In particular, the goal is to ask questions so that the observations
can provide the maximum amount of information on the answer to an
unknown evaluation question. Theorem~\ref{thm:interactionentropy}
decomposes this problem into an equivalent formulation using
interaction information, for which we seek to maximize the amount of
redundancy between the chosen questions $X_k$ and the unknown
evaluation question $\bm S$.

\begin{theorem}
  \label{thm:interactionentropy}
  Under the \ref{assm:noise} and \ref{assm:misclassifynoise}, we have
  \begin{equation}
    \label{eq:interactionentropy}
    I_k(\bm W(S);\bm Y_k(X_k)\,|\,\bm S) = I_k(\bm W(S);\bm Y_k(X_k);\bm Z_k(X_k)\,|\,\bm
    S) \leq I_k(\bm Y_k(X_k);\bm Z_k(X_k)).
  \end{equation}
\end{theorem}
\begin{proof}
  First, we use the fact that conditional mutual information is
  symmetric \citep[][p.~22]{cover1991} to get
  \begin{multline*}
    H_k\left(\bm Y_k(X_k)\,|\,\bm Z_k(X_k),\,\bm W(\bm S),\,\bm S\right) +
    H_k\left(\bm Z_k(X_k)\,|\,\bm W(\bm S),\,\bm S\right) \\
    = H_k\left(\bm Z_k(X_k)\,|\,\bm Y_k(X_k),\,\bm W(\bm S)\,,\bm S\right) +
    H_k\left(\bm Y_k(X_k)\,|\,\bm W(\bm S),\,\bm S\right),
  \end{multline*}
  and using the \ref{assm:misclassifynoise}, we see that the first
  term is equal to $H_k(\bm Y_k(X_k)\,|\,\bm Z_k(X_k))$, giving us
  \begin{multline*}
    H_k(\bm Y_k(X_k)\,|\,\bm W(\bm S),\,\bm S)
    = H_k(\bm Y_k(X_k)\,|\,\bm Z_k(X_k)) 
    + H_k(\bm Z_k(X_k)\,|\,\bm W(\bm S),\,\bm S) \\
    - H_k(\bm Z_k(X_k)\,|\,\bm Y_k(X_k),\,\bm W(\bm S),\,\bm S).
  \end{multline*}
  Subtracting both sides of the above equation from $H(\bm Y_k(X_k))$
  gives us
  \begin{equation}
    \label{eq:computeinteraction}
    I_k(\bm W(\bm S);\bm Y_k(X_k)\,|\,\bm S)
    = I_k(\bm Y_k(X_k);\bm Z_k(X_k)) - I_k(\bm Y_k(X_k);\bm
    Z_k(X_k)\,|\,\bm W(\bm S),\,\bm S)
  \end{equation}
  Now since $\bm S$ is independent of $\bm Y_k(X_k)$ and $\bm
  Z_k(X_k)$, we have $I(\bm Y_k(X_k);\bm Z_k(X_k)) = I(\bm
  Y_k(X_k);\bm Z_k(X_k)\,|\,\bm S)$, and the equality in
  \eqref{eq:interactionentropy} directly follows. The inequality is
  because the last term in \eqref{eq:computeinteraction} is
  non-negative, due to the properties of mutual information.
\end{proof}
As previously mentioned, interaction information does not have to be
non-negative. Here, the equality in \eqref{eq:interactionentropy}
implies that the interaction information is non-negative since
$I_k(\bm W(\bm S);\bm Y_k(X_k)\,|\,\bm S)$ is always
non-negative. This means that when we ask question $X_k$, observing
signal $\bm Y_k(X_k)$ yields less information when we also know the
true answer $\bm W(S)$ to another question $S$, an intuitive
result. We use Theorem~\ref{thm:interactionentropy} to relate $I_k(\bm
W_k(\bm S_k);\bm Y_k(X_k)\,|\,\bm S_k)$ to $I_k(\bm \theta;\bm
Y_k(X_k))$.


\subsection{Lower Bound on Misclassification Error}
\label{subsec:lower bnd}

We now would like to relate misclassification error to the entropy of
the posterior distribution of the linear classifier
$\bm\theta$. Theorem~\ref{thm:fanolowerbound} shows that regardless of
the estimator $\hat W_k$, one cannot reduce misclassification error
without bound unless the posterior entropy of $\bm\theta$ is reduced
as well. This is due to an important tool in information theory called
Fano's Inequality.
\begin{theorem}
  \label{thm:fanolowerbound}
  For any policy $\pi$, a lower bound for the misclassification error
  under that policy is given by
  \begin{equation*}
    \mathcal{E}_k^\pi
    \geq \frac{H(\bm W_k(\bm S_k) \,|\,\bm S_k) - I^\pi(\bm\theta;\bmYhist{k})  -  1}{\log_2 n}.
  \end{equation*}
\end{theorem}
\begin{proof}
  Suppose we have a fixed question $\bm S_k = S_k$, and let $\hat W_k$
  be any estimator of $\bm W_k(S_k)$ that is a function of history
  $\Yhist{k}$ and known assessment question $S_k$. By Fano's
  inequality, \citep[][p.~39]{cover1991}, we have
  \begin{align}
    \label{eq:fanos}
    \mathbb{P}_k(\bm W_k(S_k) \neq \hat W_k\,|\,\bm S_k = S_k) &\geq
    \frac{H_k(\bm W_k(S_k)
      \,|\,\bm S_k = S_k)-1}{\log_2 n}.\\
    \intertext{Taking an expectation over possible assessment
      questions and past history yields}
    \label{eq:fanosexp}
    \mathbb{E}^\pi\left[\mathbb{P}_k(\bm W_k(\bm S_k) \neq \hat W_k \,|\,\bm S_k)\right]
    &\geq \frac{H^\pi(\bm W_k(\bm S_k) \,|\,\bmYhist{k},\,\bm S_k)-1}{\log_2 n},
  \end{align}
  where the right side holds because of the definition of conditional
  entropy. Now we use the upper bound on $I_k(\bm W_k(\bm S_k);\bm
  Y_k(X_k) \,|\, \bm S_k)$ from Theorem~\ref{thm:interactionentropy}
  to show
  \begin{align*}
    H^\pi(\bm W_k(\bm S_k)\,|\,\bmYhist{k},\,\bm S_k)
    &= H( \bm W_k(\bm S_k)\,|\,\bm S_k ) - I^\pi(\bm W_k(\bm S_k);
    \bmYhist{k} \,|\, \bm S_k) \\
    &=  H( \bm W_k(\bm S_k)\,|\,\bm S_k ) -
    \E^\pi\left[\sum_{\ell=1}^k I_\ell(\bm W_\ell(\bm S_\ell);
      \bm Y_\ell(X_\ell))\right] \\
    &\leq H( \bm W_k(\bm S_k)\,|\,\bm S_k ) -
    \E^\pi\left[\sum_{\ell=1}^k I_\ell(\bm Z_\ell (X_\ell);
      \bm Y_\ell(X_\ell))\right] \\
    &=    H( \bm W_k(\bm S_k)\,|\,\bm S_k ) -
    \E^\pi\left[\sum_{\ell=1}^k I_\ell(\bm \theta ;
      \bm Y_\ell(X_\ell))\right] \\
    &= H( \bm W_k(\bm S_k)\,|\,\bm S_k ) - I^\pi(\bm \theta;\bmYhist{k}),
  \end{align*}
  where the penultimate equality is from
  Theorem~\ref{thm:entropyidentities}. Thus, we get
  \begin{equation}
    \mathbb{E}^\pi \left[\mathbb{P}_k(\bm W_k(\bm S_k) \neq \hat W\,|\, \bm Y_k(X_k),
      \bm S_k) \right]
    \geq \frac{H(\bm W_k(\bm S_k) \,|\,\bm S_k) - I^\pi(\bm\theta;\bmYhist{k})  -  1}{\log_2 n},
  \end{equation}
  and the result follows.
\end{proof}
The bound does not provide any insight if $H(\bm
W_k(\bm S_k)\,|\,\bmYhist{k}),\bm S_k) < 1$ since the lower bound
would be negative. This is most problematic when $n=2$, in which case,
the Shannon entropy of $\bm W_k$ is bounded above by one bit. However,
if the conditional entropy of $\bm W_k(S_k)$ after observing signal
$\bm Y_k(X_k)$ is still significantly large, the misclassification
error will not be reduced past a certain threshold.

There are some interesting conclusions that can be drawn from the
lower bound. First, $H(\bm W_k(\bm S_k) \,|\,\bm S_k)$ can be viewed
as a constant that describes the problem complexity, representing the
expected entropy of evaluation question $\bm S_k$. The lower bound is
a linear function with respect to the mutual information of linear
classifier $\bm \theta$ and the observation history $\bmYhist{k}$.

We can use this result to bound both the knowledge gradient policy and
the fully adaptive optimal policy from
below. Corollary~\ref{cor:policybounds} below leverages
Theorem~\ref{thm:fanolowerbound} to estimate the optimality gap of
knowledge gradient from the optimal policy.
\begin{corollary}
  \label{cor:policybounds}
  Under noise channel $f$ with channel capacity $C(f)$, the optimal
  misclassification error under the optimal policy after asking $k$
  comparative questions is bounded by
  \begin{equation*}
    \frac{H(\bm W_k(\bm S_k) \,|\,\bm S_k) - C(f)\cdot k
      -  1}{\log_2 n} \leq \mathcal{E}_k^* \leq \mathcal{E}_k^{KG}.
  \end{equation*}
\end{corollary}
Of course, there is a fairly significant gap in the lower bound, since
the misclassification errors are non-negative, and yet the lower bound
is linear. The gap comes from the second inequality in
\eqref{eq:interactionentropy}, and this upper bound essentially throws
out the redundant information about possible evaluation questions
learned by previous user responses. Nonetheless, it tells us that
posterior entropy reduction is necessary for misclassification error
reduction.

\section{Computational Results}
\label{sec:computation}

In the following subsections, we present computational results from
simulated responses using vectorizations of real
alternatives. Section~\ref{subsec:sim method} discusses our approach and
methodology for the numerical experiments, and Section~\ref{subsec:sim
  results} gives the results of the computational studies and provides
insights regarding the performance of the entropy pursuit and
knowledge gradient policies.

\subsection{Methodology}
\label{subsec:sim method}  

As an alternative space, we use the 13,108 academic papers on
\href{www.arxiv.org}{arXiv.org} from the condensed matter archive
written in 2014. The information retrieval literature is rife with
different methods on how to represent a document as a vector,
including bag of words, term frequency inverse document frequency
\citep{salton1986introduction}, and word2vec
\citep{goldberg2014word2vec}, along with many others \citep[for an
  overview of such methods, see][]{raghavan2008introduction}. In
practice, the method for vectorizing the alternatives is critical; if
the vectors do not sufficiently represent the alternatives, any
recommendation system or preference elicitation algorithm will have
trouble. For the numerical experiments, we elected to use a vector
representation derived from Latent Dirichlet Allocation (LDA) as
described by \citet*{blei2003latent}. The resulting feature vectors
are low-dimensional and dense. Since we cannot compute the posterior
distribution analytically, we resort to sampling instead, and the
low-dimensional LDA vectors allow for more efficient sampling.

With any method that utilizes Bayesian inference, it is important to
have enough structure that allows for an efficient sampling scheme
from the resulting posterior distributions. The benefit of having the
simple update of up-weighting and down-weighting polytopes is that the
sampling scheme becomes quite easy. We use a hit-and-run sampler as
described and analyzed by \citet*{lovasz2003hit} that chooses a
direction uniformly from the unit sphere, then samples from the
one-dimensional conditional distribution of the next point lying on
that line. Now, re-weighting polytopes turns into re-weighting line
segments. If it is easy sample points from the conditional
distribution of lying on a given line, hit-and-run is an efficient way
of sampling. We use a multivariate normally distributed prior because
it allows for both computational tractability for sampling from this
conditional distribution as well as a natural representation of prior
information.

To select the hyperparameters for the prior, we sample academic papers
and fit a multivariate normal distribution to this sample. Assuming
users' linear classifiers have the same form and interpretation as an
vector representation is not reasonable in general. However, in the
case of academic papers, authors are also readers, and so the content
in which the users are interested is closely related to the content
they produce. Therefore, in this situation, it is reasonable to assume
that a user's parameterization of preferences lives in the same space
as the parameterization of the feature set. This is not necessarily
the case for other types of alternatives, and even if it were, using
feature vectors to model preference vectors may not be the best
choice. That being said, there are many ways to initialize the
prior. If one has a history of past user interaction with
alternatives, one could estimate the linear preference vector for each
user using an expectation maximization scheme, and fit a mixed normal
prior to the empirical distribution of estimated linear classifiers,
as done by \citet*{chen2016bayesian}. Since the focus here is to
compare the performance of the two algorithms of interest, our choice
for initializing the prior is sufficient.

\subsection{Cross-Metric Policy Comparison}
\label{subsec:sim results}

We first compare the entropy pursuit and knowledge gradient policies
while varying the number of presented alternatives. Due to the large
set of alternatives, it is computationally intractable to choose
questions that optimally follow either policy, so alternatives are
subsampled from $\X$ and we approximate both policies using the
alternatives from the subsample. If $N$ alternatives are subsampled,
then the approximate entropy pursuit policy requires exhaustively
optimizing over combinations of alternatives (permutations if the noise
channel not symmetric), and hence will require maximizing over
$\binom{N}{m}$ subsets. On the other hand, the knowledge gradient
policy requires comparing $\binom{N}{m}$ informative questions $X$
with $\binom{N}{n}$ assessment questions $S$, and thus requires
estimating $\binom{N}{m} \binom{N}{n}$ quantities. Already, this
implies that if the computational budget per question is fixed for
both algorithms, one can afford a polynomially larger subsample for
entropy pursuit than for knowledge gradient. For example, in the case
where $m=n=2$, a computational budget that allows a subsample of
$N=15$ alternatives for the knowledge gradient policy would allow the
entropy pursuit policy a subsample size of $N' = 149$. However, rather
than fixing a computational budget for both policies at each step, we
allow both policies the same number of subsamples, setting $N=15$ for
both policies and all sets of parameters. We do this to allow for a
more straightforward comparison of the two policies, although further
computational studies should study their performance under a fixed
computational budget. Lastly, the numerical study in this paper fixes
$n=2$. We make this decision because any larger values of $n$ will
make the computations prohibitively expensive, and it is not clear
that larger values of $n$ will provide any additional benefit.

\begin{figure}[t]
  \centering
  \includegraphics[width=\textwidth]{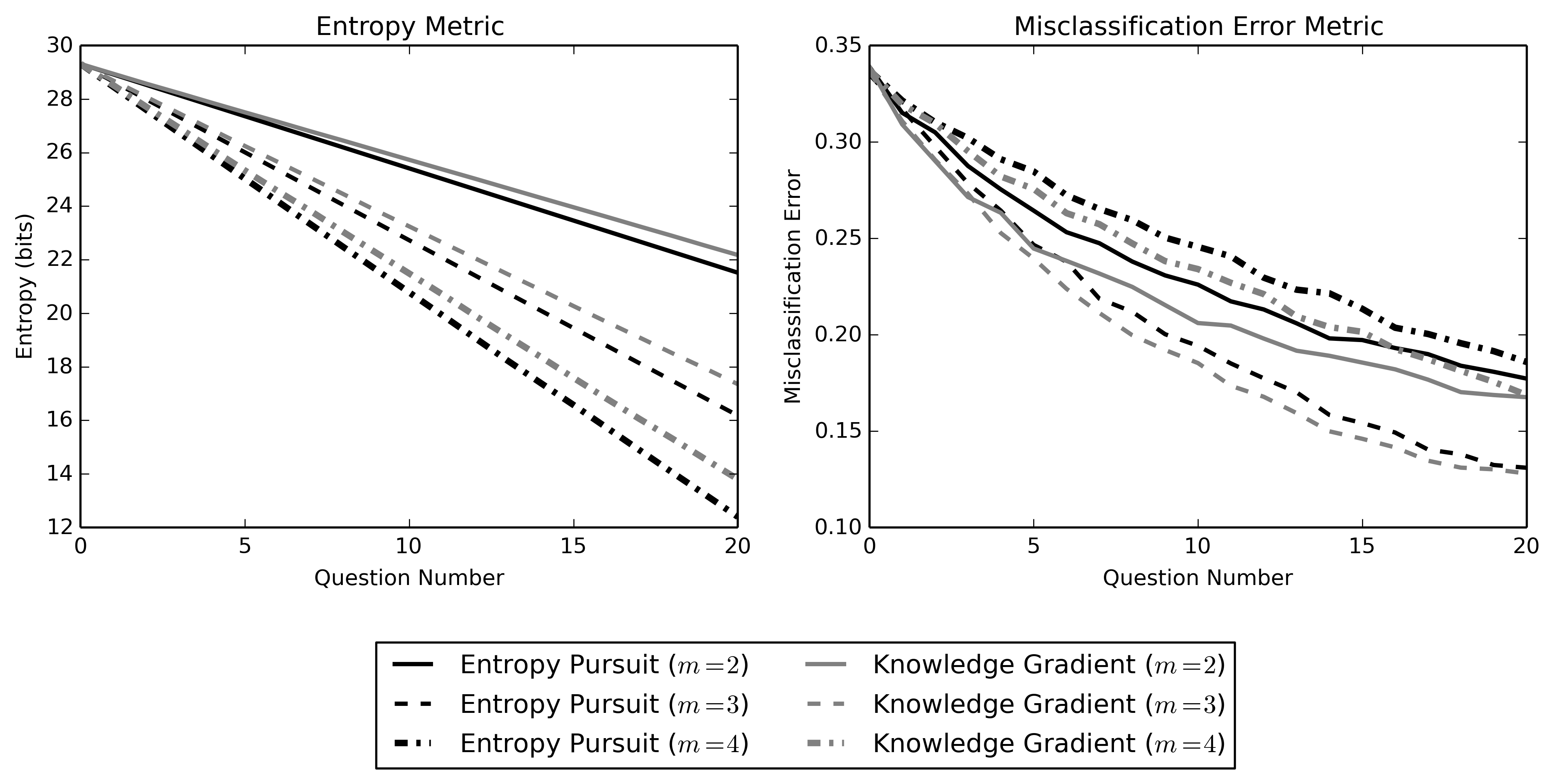}
  \caption{Comparison of average performance of the entropy pursuit
    and knowledge gradient policies under a symmetric noise channel
    ($\alpha=0.7$), simulated and averaged with 100 sample
    paths. Estimates are accurate to $\pm 0.007$ for
    misclassification error and $\pm 0.06$ bits for entropy.}
  \label{fig:envskg}
\end{figure}

Figure~\ref{fig:envskg} compares the entropy pursuit and knowledge
gradient policies by varying $m$ and fixing other parameters to
reflect a low-noise, low prior information scenario. As expected, each
algorithm performs better on their respective metrics for a fixed
number of provided alternatives $m$. However, a more surprising
conclusion is the similarity in performance of the two algorithms for
any fixed $m$ for \textit{both metrics}. This suggests that the price
to pay for switching from the knowledge gradient policy to the entropy
pursuit policy is small compared to the gain in computational
efficiency. In fact, if the computational budget for each question
were fixed, one would be able to subsample many more alternatives to
compute the entropy pursuit policy compared to the knowledge gradient
policy, and it is very likely the former would out-perform the latter
in this setting. To see if this occurrence takes place in other
scenarios, such as those with higher noise and a more informative
prior, one can consult Figure~\ref{fig:situational}. Again, for all
the different parameter settings, both policies perform similarly.

\begin{figure}[t]
  \centering
  \includegraphics[width=\textwidth]{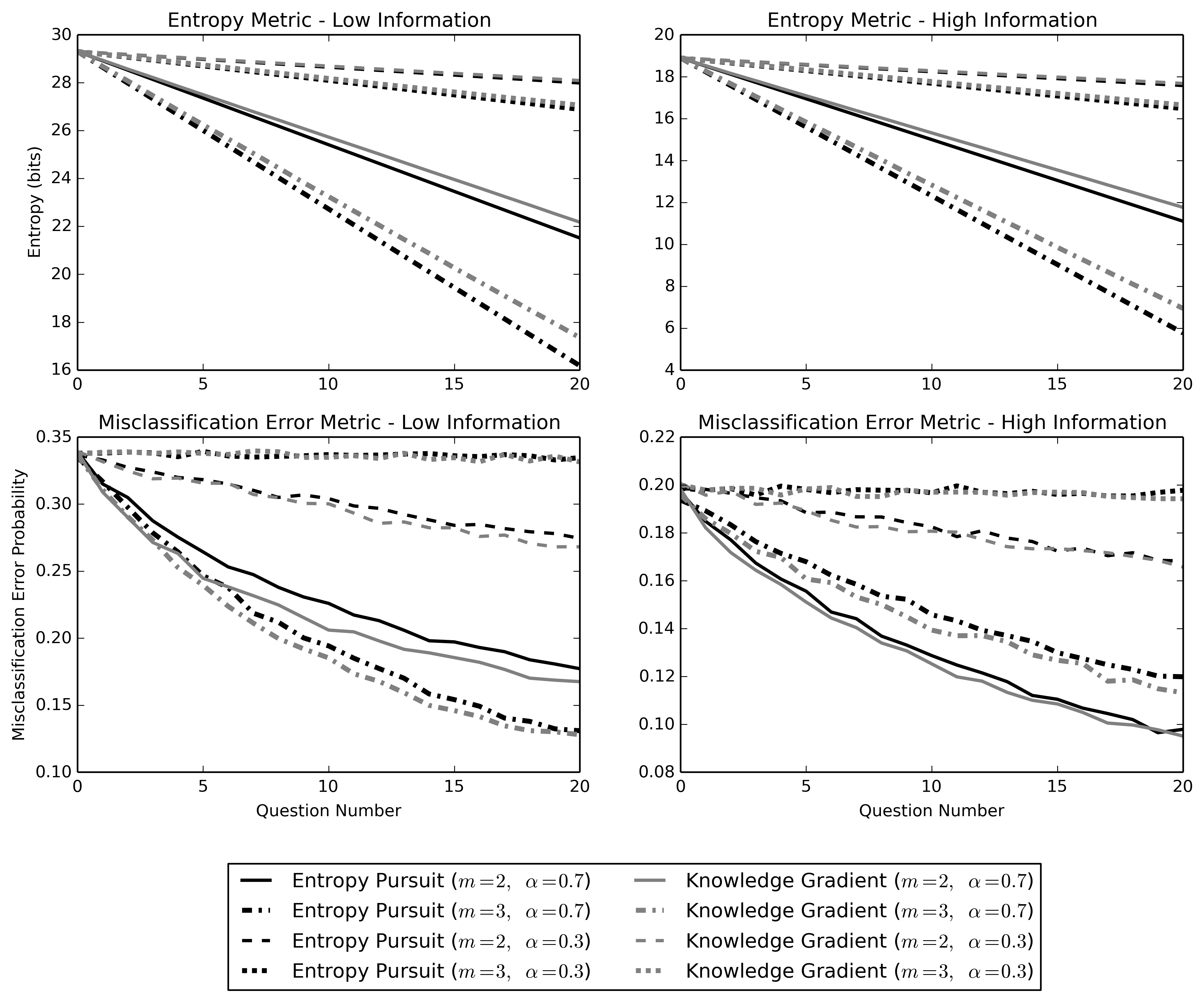}
  \caption{Comparison of the entropy pursuit and knowledge gradient
    policies under a symmetric noise channel for various levels of
    noise and prior information, simulated and averaged with 100
    sample paths. Estimates are accurate to $\pm 0.007$ for
    misclassification error and $\pm 0.06$ bits for entropy.}
  \label{fig:situational}
\end{figure}

Another interesting aspect of the computational results are the
effects of the parameters on the performance of the two
policies. Differential entropy predictably decreases faster when more
alternatives are presented to the user. In the case of a symmetric
noise channel, increasing $m$ only increases the channel capacity for
a fixed noise level $\alpha$. From the perspective of minimizing
posterior entropy, this makes sense because offering more alternatives
at each time epoch should theoretically allow one to refine the
posterior distribution faster. However, in reality, the noise channel
most likely varies with the number of offered alternatives $m$, where
the quality of the noise channel degrades as $m$ grows. In the most
extreme example, offering too many alternatives to the user will
result in a phenomenon called ``decision paralysis,'' where the user's
responses will not contain useful information about her
preferences. In this case, the model is not capturing the added
uncertainty, and focusing on posterior entropy as a performance metric
may be misleading.

In contrast, the knowledge gradient policy captures this intuition,
since pairwise comparisons decrease misclassification error faster
than three-way comparisons in the cases of high noise or a highly
informative prior. In fact, three-way comparisons only prevail in a
low-noise, low prior information scenario, which is fairly
optimistic. Both policies under three-way comparisons were aggressive,
and in the high-noise case, they fail to learn anything at all about
the user's preferences. In practice, it will be necessary to estimate
parameters for the noise channel in order to choose the correct value
of $m$. For now, it suffices to say that pairwise comparisons are
robust and reliable.

\section{Conclusion}
\label{sec:conclusion}

In this paper, we analyze the problem of eliciting a given user's
preferences by adaptively querying the user with choice-based
questions. We formulate this problem in a sequential active learning
setting, where a user's preferences are governed by an unknown linear
classifier, and the observed responses are perturbed by noise.  We
assume the underlying observation model where noise does not depend on
the underlying preferences. Under this regime, we show that the
differential entropy of the posterior distribution of this linear
classifier can be reduced linearly with respect to the number of
questions posed. Further, there exists an optimal predictive
distribution that allows this optimal linear rate to be attained.  We
provide sensitivity results that show the entropy reduction is close
to maximal when the actual predictive distribution of a given question
is close to optimal in $L_2$ distance.

On the problem of appropriately choosing the alternatives: when the
set of alternatives has non-empty interior, we provide a construction
to find a question that achieves the linear lower bound to a constant
multiplicative factor, and exactly for predictive distributions when
$\max\{ u_* \} = 1/2$ for pairwise comparisons or $\max\{u_*\} < 1/2$
for multi-way comparisons. When the set of alternatives is large but
finite, we have demonstrated through simulation experiments that one
can find questions that consistently yield a linear decrease in
differential entropy, and this rate is reasonably close to optimal.

In addition to focusing on differential entropy, we consider
misclassification error as an alternative metric that more intuitively
captures the knowledge one has for a user's preferences. Using Fano's
inequality, a classic result in the field of information theory, we
show the performance of the optimal policy with respect to this metric
is bounded below by a linear function in posterior entropy, suggesting
a relationship between entropy-based and misclassification error-based
policies. Our computational results largely confirm this, as the
entropy pursuit policy and the knowledge gradient policy perform
similarly in a variety of scenarios. For this reason, and the fact
that the knowledge gradient requires a significantly larger
computational budget, entropy pursuit is preferred for adaptive
choice-based active preference learning.

Although the paper assumes that the number of alternatives $m$ is
constant with respect to time, this can be relaxed with a word of
caution. From the perspective of entropy, it is always beneficial to
increase $m$, which can be misleading. Thus, if $m$ is allowed to vary
with time, one should not use entropy pursuit to choose $m$, and
should use another method to select the number of alternatives to
present to the user. This may be done by fixing a static sequence
$m_k$ in advance, or the parameter could be adjusted adaptively by
another policy in tandem with entropy pursuit. Both approaches would
most likely require extensive precomputation, since the geometry of
the space of alternatives would heavily affect any policy governing
$m$. Similar is the case of when a suitable prior for the user is not
known. In practice, this would also dictate the need for a
preprocessing step, perhaps fitting a Gaussian mixture to a population
of estimated linear classifiers
\citep*[see]{chen2016bayesian}. Regardless, this paper motivates the
use of entropy pursuit in adaptive choice-based preference
elicitation, as well as the study of its effectiveness using
historical user responses and experimentation.

\acks{Peter Frazier was partially supported by NSF CAREER
  CMMI-1254298, NSF CMMI-1536895, NSF IIS-1247696, AFOSR
  FA9550-12-1-0200, AFOSR FA9550-15-1-0038, AFOSR FA9550-16-1-0046,
  and DMR-1120296. \\

  Shane Henderson was supported in part by NSF CMMI-1537394.}

\newpage
\bibliography{citations}

\begin{thebibliography}{29}
\providecommand{\natexlab}[1]{#1}
\providecommand{\url}[1]{\texttt{#1}}
\expandafter\ifx\csname urlstyle\endcsname\relax
  \providecommand{\doi}[1]{doi: #1}\else
  \providecommand{\doi}{doi: \begingroup \urlstyle{rm}\Url}\fi

\bibitem[Ailon(2012)]{ailon2012active}
Nir Ailon.
\newblock An active learning algorithm for ranking from pairwise preferences
  with an almost optimal query complexity.
\newblock \emph{Journal of Machine Learning Research}, 13\penalty0
  (Jan):\penalty0 137--164, 2012.

\bibitem[Bernardo(1979)]{bernardo1979expected}
Jos{\'e}~M Bernardo.
\newblock Expected information as expected utility.
\newblock \emph{The Annals of Statistics}, pages 686--690, 1979.

\bibitem[Blei et~al.(2003)Blei, Ng, and Jordan]{blei2003latent}
David~M Blei, Andrew~Y Ng, and Michael~I Jordan.
\newblock Latent {D}irichlet {A}llocation.
\newblock \emph{Journal of machine Learning research}, 3\penalty0
  (Jan):\penalty0 993--1022, 2003.

\bibitem[Brochu et~al.(2010)Brochu, Brochu, and de~Freitas]{brochu2010bayesian}
Eric Brochu, Tyson Brochu, and Nando de~Freitas.
\newblock A {B}ayesian interactive optimization approach to procedural
  animation design.
\newblock In \emph{Proceedings of the 2010 ACM SIGGRAPH/Eurographics Symposium
  on Computer Animation}, pages 103--112. Eurographics Association, 2010.

\bibitem[Chen and Frazier(2016)]{chen2016bayesian}
Bangrui Chen and Peter Frazier.
\newblock The {B}ayesian linear information filtering problem.
\newblock \emph{arXiv preprint arXiv:1605.09088}, 2016.

\bibitem[Cover(1991)]{cover1991}
Thomas~M Cover.
\newblock \emph{Elements of information theory}.
\newblock John Wiley, 1991.

\bibitem[Donoho and Gasko(1992)]{donoho1992breakdown}
David~L Donoho and Miriam Gasko.
\newblock Breakdown properties of location estimates based on halfspace depth
  and projected outlyingness.
\newblock \emph{The Annals of Statistics}, pages 1803--1827, 1992.

\bibitem[Dzyabura and Hauser(2011)]{dzyabura2011active}
Daria Dzyabura and John~R Hauser.
\newblock Active machine learning for consideration heuristics.
\newblock \emph{Marketing Science}, 30\penalty0 (5):\penalty0 801--819, 2011.

\bibitem[F{\"u}rnkranz and H{\"u}llermeier(2003)]{furnkranz2003pairwise}
Johannes F{\"u}rnkranz and Eyke H{\"u}llermeier.
\newblock Pairwise preference learning and ranking.
\newblock In \emph{European conference on machine learning}, pages 145--156.
  Springer, 2003.

\bibitem[Gallager(1968)]{gallager1968information}
Robert~G Gallager.
\newblock \emph{Information theory and reliable communication}, volume~2.
\newblock Springer, 1968.

\bibitem[Goldberg and Levy(2014)]{goldberg2014word2vec}
Yoav Goldberg and Omer Levy.
\newblock word2vec explained: deriving mikolov et al.'s negative-sampling
  word-embedding method.
\newblock \emph{arXiv preprint arXiv:1402.3722}, 2014.

\bibitem[Green and Srinivasan(1978)]{green1978conjoint}
Paul~E Green and Venkatachary Srinivasan.
\newblock Conjoint analysis in consumer research: issues and outlook.
\newblock \emph{Journal of consumer research}, 5\penalty0 (2):\penalty0
  103--123, 1978.

\bibitem[Houlsby et~al.(2011)Houlsby, Husz{\'a}r, Ghahramani, and
  Lengyel]{houlsby2011bayesian}
Neil Houlsby, Ferenc Husz{\'a}r, Zoubin Ghahramani, and M{\'a}t{\'e} Lengyel.
\newblock {B}ayesian active learning for classification and preference
  learning.
\newblock \emph{arXiv preprint arXiv:1112.5745}, 2011.

\bibitem[Jedynak et~al.(2012)Jedynak, Frazier, Sznitman,
  et~al.]{jedynak2012twenty}
Bruno Jedynak, Peter~I Frazier, Raphael Sznitman, et~al.
\newblock Twenty questions with noise: Bayes optimal policies for entropy loss.
\newblock \emph{Journal of Applied Probability}, 49\penalty0 (1):\penalty0
  114--136, 2012.

\bibitem[Lindley(1956)]{lindley1956measure}
Dennis~V Lindley.
\newblock On a measure of the information provided by an experiment.
\newblock \emph{The Annals of Mathematical Statistics}, pages 986--1005, 1956.

\bibitem[Louviere et~al.(2000)Louviere, Hensher, and Swait]{louviere2000stated}
Jordan~J Louviere, David~A Hensher, and Joffre~D Swait.
\newblock \emph{Stated choice methods: analysis and applications}.
\newblock Cambridge University Press, 2000.

\bibitem[Lov{\'a}sz and Vempala(2003)]{lovasz2003hit}
L{\'a}szl{\'o} Lov{\'a}sz and Santosh Vempala.
\newblock Hit-and-run is fast and fun.
\newblock \emph{Microsoft Research}, 2003.

\bibitem[MacKay(1992)]{mackay1992information}
David~JC MacKay.
\newblock Information-based objective functions for active data selection.
\newblock \emph{Neural computation}, 4\penalty0 (4):\penalty0 590--604, 1992.

\bibitem[Maldonado et~al.(2015)Maldonado, Montoya, and
  Weber]{maldonado2015advanced}
Sebasti{\'a}n Maldonado, Ricardo Montoya, and Richard Weber.
\newblock Advanced conjoint analysis using feature selection via support vector
  machines.
\newblock \emph{European Journal of Operational Research}, 241\penalty0
  (2):\penalty0 564--574, 2015.

\bibitem[Nowak(2011)]{nowak2011geometry}
Robert~D Nowak.
\newblock The geometry of generalized binary search.
\newblock \emph{Information Theory, IEEE Transactions on}, 57\penalty0
  (12):\penalty0 7893--7906, 2011.

\bibitem[Raghavan and Sch{\"u}tze(2008)]{raghavan2008introduction}
P~Raghavan and H~Sch{\"u}tze.
\newblock Introduction to information retrieval, 2008.

\bibitem[Salton and McGill(1986)]{salton1986introduction}
Gerard Salton and Michael~J McGill.
\newblock Introduction to modern information retrieval.
\newblock 1986.

\bibitem[Saure and Vielma(2016)]{saure2016ellipsoidal}
Denis Saure and Juan~Pablo Vielma.
\newblock Ellipsoidal methods for adaptive choice-based conjoint analysis.
\newblock \emph{preprint SSRN 2798984}, 2016.

\bibitem[Schapire and Singer(1998)]{schapire1998learning}
William W Cohen Robert~E Schapire and Yoram Singer.
\newblock Learning to order things.
\newblock \emph{Advances in Neural Information Processing Systems}, 10\penalty0
  (451):\penalty0 24, 1998.

\bibitem[Shannon and Weaver(1948)]{shannon1948mathematical}
Claude~E Shannon and Warren Weaver.
\newblock \emph{The mathematical theory of communication}.
\newblock University of Illinois press, 1948.

\bibitem[Toubia et~al.(2004)Toubia, Hauser, and Simester]{toubia2004polyhedral}
Olivier Toubia, John~R Hauser, and Duncan~I Simester.
\newblock Polyhedral methods for adaptive choice-based conjoint analysis.
\newblock \emph{Journal of Marketing Research}, 41\penalty0 (1):\penalty0
  116--131, 2004.

\bibitem[Toubia et~al.(2007)Toubia, Hauser, and
  Garcia]{toubia2007probabilistic}
Olivier Toubia, John Hauser, and Rosanna Garcia.
\newblock Probabilistic polyhedral methods for adaptive choice-based conjoint
  analysis: Theory and application.
\newblock \emph{Marketing Science}, 26\penalty0 (5):\penalty0 596--610, 2007.

\bibitem[Tukey(1975)]{tukey1975mathematics}
John~W Tukey.
\newblock Mathematics and the picturing of data.
\newblock In \emph{Proceedings of the international congress of
  mathematicians}, volume~2, pages 523--531, 1975.

\bibitem[Yu et~al.(2012)Yu, Goos, and Vandebroek]{yu2012comparison}
Jie Yu, Peter Goos, and Martina Vandebroek.
\newblock A comparison of different {B}ayesian design criteria for setting up
  stated preference studies.
\newblock \emph{Transportation Research Part B: Methodological}, 46\penalty0
  (7):\penalty0 789--807, 2012.

\end{thebibliography}

\end{document}